\documentclass{article}

\usepackage{microtype}
\usepackage{graphicx}
\usepackage{subfigure}
\usepackage{booktabs} 

\usepackage{hyperref}

\usepackage[accepted]{icml2023}

\usepackage{amsmath}
\usepackage{amssymb}
\usepackage{mathtools}
\usepackage{amsthm}

\usepackage[capitalize,noabbrev]{cleveref}

\theoremstyle{plain}
\newtheorem{theorem}{Theorem}[section]
\newtheorem{proposition}[theorem]{Proposition}

\theoremstyle{definition}
\newtheorem{definition}[theorem]{Definition}

\theoremstyle{remark}

\usepackage[disable,textsize=tiny]{todonotes}

\usepackage{booktabs} 

\usepackage[algo2e,ruled,vlined,linesnumbered]{algorithm2e}
\SetKwComment{Comment}{$\triangleright$\ }{}

\def\eqref#1{Equation~\ref{#1}}			
\def\figref#1{Figure~\ref{#1}}			
\def\tabref#1{Table~\ref{#1}}			
\def\secref#1{Section~\ref{#1}}			
\def\algref#1{Algorithm~\ref{#1}}		
\def\dffref#1{Definition~\ref{#1}}		

\newcommand\filledcirc{\ensuremath{{\color{white}\bullet}\mathllap{\circ}}}

\newcommand{\circleftarrow}{\mathrel{\leftarrow}\!\!\!\filledcirc}
\newcommand{\circleftrightarrow}{\mathrel{\circ\mbox{--}\circ}}

\newcommand{\nodes}{\mathbf{V}}
\newcommand{\obs}{\mathbf{O}}
\newcommand{\lat}{\mathbf{L}}

\newcommand{\pag}{\mathcal{P}}

\DeclareMathOperator{\MAG}{\mathcal{M}}
\DeclareMathOperator{\homology}{\mathrm{homology}}

\DeclareMathOperator{\ICDSep}{\mathbf{ICD-Sep}}
\DeclareMathOperator{\DSepop}{\mathbf{D-Sep}}
\DeclareMathOperator{\PDSepop}{\mathbf{Possible-D-Sep}}

\newcommand{\DSEPset}[2]{\DSepop(#1, #2)}
\newcommand{\PDSEPset}[2]{\PDSepop(#1, #2)}
\newcommand{\PDSPath}[3]{\Pi_{#2}(#1, #3)}

\newcommand{\assign}{\leftarrow}

\DeclareMathOperator{\sindep}{Ind}  

\makeatletter
\newcommand*{\indep}{%
  \mathbin{%
    \mathpalette{\@indep}{}%
  }%
}
\newcommand*{\nindep}{%
  \mathbin{
    \mathpalette{\@indep}{\not}
  }%
}
\newcommand*{\@indep}[2]{%
  \sbox0{$#1\perp\m@th$}
  \sbox2{$#1=$}
  \sbox4{$#1\vcenter{}$}
  \rlap{\copy0}
  \dimen@=\dimexpr\ht2-\ht4-.2pt\relax
  \kern\dimen@
  {#2}%
  \kern\dimen@
  \copy0 
} 
\makeatother

\newcommand{\tsicd}{TS-ICD~}

\icmltitlerunning{From Temporal to Contemporaneous Iterative Causal Discovery in the Presence of Latent Confounders}

\begin{document}

\twocolumn[
\icmltitle{From Temporal to Contemporaneous Iterative Causal Discovery \\ in the Presence of Latent Confounders}



\icmlsetsymbol{equal}{*}

\begin{icmlauthorlist}
\icmlauthor{Raanan Y. Rohekar}{comp}
\icmlauthor{Shami Nisimov}{comp}
\icmlauthor{Yaniv Gurwicz}{comp}
\icmlauthor{Gal Novik}{comp}
\end{icmlauthorlist}

\icmlaffiliation{comp}{Intel Labs}

\icmlcorrespondingauthor{Raanan Y. Rohekar}{raanan.yehezkel@intel.com}

\icmlkeywords{Causal Discovery, Time-series, Latent Confounders, FCI}

\vskip 0.3in
]



\printAffiliationsAndNotice{}  

\begin{abstract}
We present a constraint-based algorithm for learning causal structures from observational time-series data, in the presence of latent confounders. We assume a discrete-time, stationary structural vector autoregressive process, with both temporal and contemporaneous causal relations. One may ask if temporal and contemporaneous relations should be treated differently. The presented algorithm gradually refines a causal graph by learning long-term temporal relations before short-term ones, where contemporaneous relations are learned last. This ordering of causal relations to be learnt leads to a reduction in the required number of statistical tests. We validate this reduction empirically and demonstrate that it leads to higher accuracy for synthetic data and more plausible causal graphs for real-world data compared to state-of-the-art algorithms.
\end{abstract}

\section{Introduction}

Automated discovery of causal structures from observational time-series is an important goal in empirical science \citep{pearl2010introduction, spirtes2010introduction, peters2017}. In many cases, the use of interventional experiments to reason about the underlying causal structure is unethical, impossible due to technical reasons, or expensive. Another challenge is to reason about the structure when latent confounders may exist. In this paper we consider a set of variables, measured simultaneously and repeatedly over time, resulting in discrete-time series. It is desired to identify temporal (lagged), as well as contemporaneous causal relations among the variables. A temporal causal relation exists if the value of a variable effects its own future values, or the future values of other variables. Alternatively, a latent (hidden) variable may effect the values of two variables, measured at two distinct time stamps. Contemporaneous causal relations are causal relations between simultaneously measured variables (within the same measurement time-window), or between a latent variable and those variables. 
In this work we assume stationary, structural vector auto-regression (SVAR), with latent confounders as the data-generating process, and the sole presence of observational data. In addition, we assume faithfulness \citep{spirtes2000} and causal Markov assumptions \citep{pearl2009causality}. We do not assume any parametric assumptions, such as linear relations and normally distributed variables. We assume the causal structure is a DAG over observed and latent variables, where stationarity is represented by a graph over observed variables, containing temporal and contemporaneous relations, repeating in time.
Under these assumptions, constraint-based causal-discovery algorithms perform a series of conditional independence (CI) tests for determining if two variables are independent conditioned on a set of variables, and then detect the entailed causal relations \citep{spirtes2000, zhang2008completeness, pearl1991theory}. In general, the true underlying causal graph cannot be recovered from pure observational data. Instead, an equivalence class of the underlying graph is sought, such that it includes every graph that cannot be refuted given the observed data (or the results of CI-tests). Interestingly, it was recently shown that an equivalence class can be used for causal inference \citep{zhang2008causal, jaber2018causal, jaber2019causal} and even for learning deep neural network structures \citep{rohekar2018constructing, rohekar2019modeling} and interpreting self-attention \citep{nisimov2022clear}.
\citet{richardson2002ancestral} derived ancestral graphs for representing the equivalence class in the presence of latent confounders, where they showed that these graphs are closed under marginalization and conditioning. In this paper, we consider learning this equivalence class for representing temporal and contemporaneous relations in the SVAR setting with finite order---a dynamic partial-ancestral-graph \citep{malinsky2018causal}.

\section{Related Work}

Constraint-based algorithms for causal discovery were mostly pioneered by \citet{spirtes2000} with the PC and FCI algorithms, and by \citet{pearl1991theory}. They are generally sound and complete when a perfect CI-test is used (e.g., a statistical test in the large sample limit). That is, they return an equivalence class in which no member can be refuted and no member is missing under the assumptions on which they rely. However, often a limited dataset of observations is available for which statistical tests of conditional independence are prone to errors. Thus, these algorithms commonly seek to reduce the total number of conditional independence (CI) tests \citep{colombo2012learning, claassen2013learning, nisimov2021improving}, or measure their uncertainty \citep{rohekar2018bayesian}. Recently, \citet{gerhardus2020high} proposed a method for improving the accuracy of CI-tests in the linear-Gaussian case, by augmenting the conditioning set with parents of the tested nodes. However, since many statistical tests, such as G-square and conditional mutual information, suffer from the curse-of-dimensionality, it is often more important to reduce CI tests having large conditioning sets. \citet{spirtes2000} proposed starting off with a fully-connected graph and iteratively increasing the size of the conditioning set (PC and FCI algorithms). Thus, in each iteration the graph becomes sparser, and as the conditioning set sizes increase, fewer CI-tests are required. However, in the presence of latent confounders, the FCI algorithm requires applying this iterative process twice, resulting in a possibly sub-optimal result. \citet{rohekar2021iterative}, using the principles underlying FCI, proposed the ICD algorithm that requires a single iterative learning loop, leading to significantly fewer CI tests. 

Based on these approaches, several algorithms for learning from time-series were proposed in the past decade. \citet{entner2010causal} proposed tsFCI, an adaptation of FCI \citep{spirtes1999algorithm} to time-series, but excluded the possibility of contemporaneous causal relations. \citet{malinsky2018causal} proposed SVAR-FCI, another adaptation of FCI that learns both contemporaneous and temporal relations. In the past years, several improvements to the FCI algorithm were proposed, among those is the RFCI algorithm that requires significantly fewer CI-tests than FCI. \citet{gerhardus2020high}, demonstrated the applicability of this algorithm for learning time-series, and called it SVAR-RFCI. \citet{gerhardus2020high} also proposed the novel LPCMCI algorithm and demonstrated empirically that it achieves state-of-the-art accuracy. However, LPCMCI employs larger conditioning sets compared to other algorithms, which may hinder its accuracy when using a CI-test that suffers from the curse-of-dimensionality. Thus, it is desired to both reduce the total number of CI-tests, and reduce the conditioning set sizes, while maintaining soundness and completeness.

Under causal sufficiency and for learning contemporaneous relations, \citet{yehezkel2009rai} proposed to recursively split the graph nodes, during learning, into two types: exogenous and endogenous, calling disconnected groups `autonomy'. Then, they start off with learning edges connecting exogenous to endogenous nodes, and only then to learn the sub-graph over the endogenous nodes (refinement of sub-graphs over exogenous nodes are called recursively). They demonstrate that this approach allows further reducing the number of CI-tests having large conditioning sets. Although the usefulness of this approach was recently demonstrated in other settings \citep{pmlr-v186-sugahara22a}, it is not clear how to extend this approach to support the presence of latent variables, and to SVAR time-series data setting for learning dynamic partial-ancestral-graphs.

In this work, we employ the approach of autonomy identification (exogenous-endogenous split), and the approach of learning with a single iterative causal discovery loop of the ICD algorithm \citep{rohekar2021iterative}, and present the TS-ICD algorithm that efficiently learns a dynamic partial-ancestral-graph from time-series data. Interestingly, as it is shown in this paper, autonomy identification in SVAR setting leads TS-ICD to learn long-range temporal relations before learning short-range ones, and learn contemporaneous relations after learning temporal relations. This is different from a simple adaptation of the ICD algorithm to time-series data, which has no order upon which CI tests are called (we also empirically demonstrate this difference in an ablation study).

\section{Autonomy Identification for Iterative Causal Discovery from Time-Series}

First, we provide preliminaries. Then we derive TS-ICD that adapts the learning paradigm called RAI, recursive autonomy identification, \citep{yehezkel2009rai} for causal discovery under the assumptions in this paper.

\begin{figure}[t]
\centering
\includegraphics[width=0.9\columnwidth]{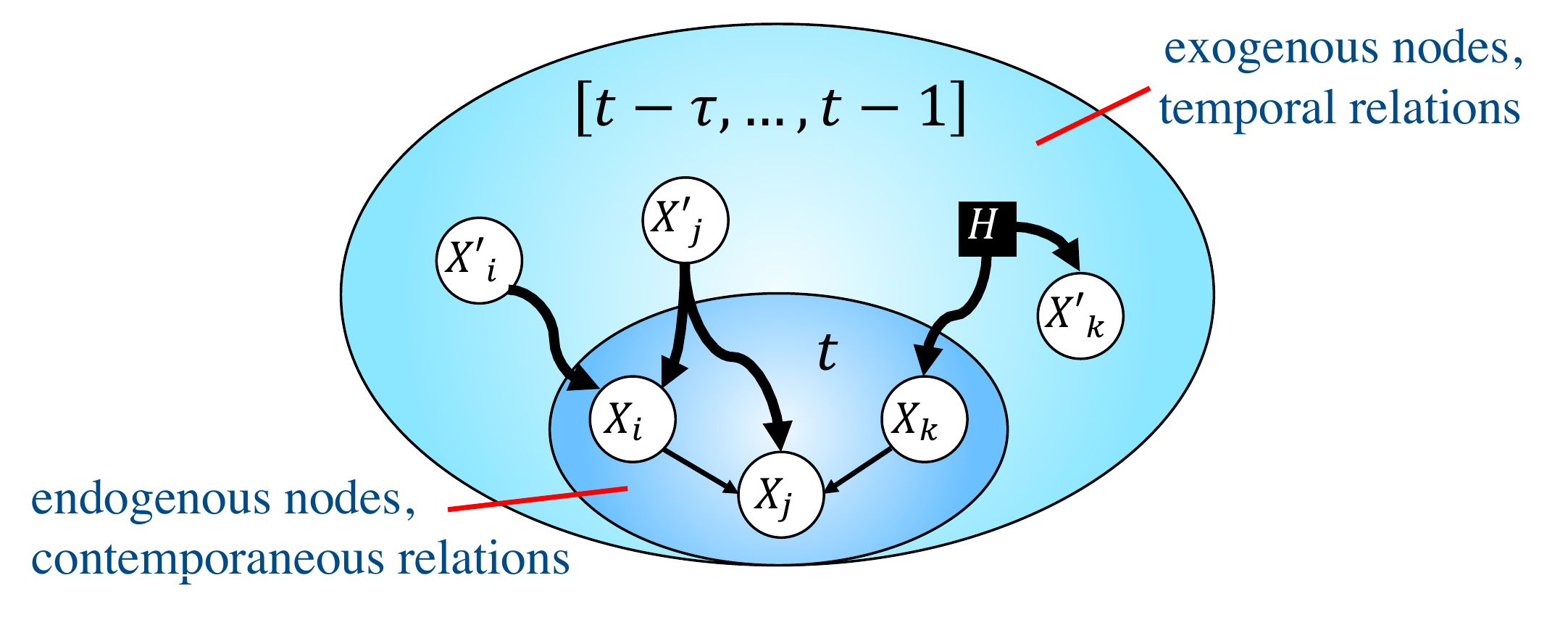}
\caption{A illustration of the distinction between endogenous (in the inner oval) and exogenous (outside the inner oval) nodes, and the temporal (thick curved arrows) and contemporaneous (straight thin arrows) causal relations. Endogenous nodes $\{X_i,X_j,X_k\}$ are measurements at time $t$, and exogenous nodes $\{X_i,X_j,X_k\}$ are past measurements, having time-stamps within $[t-\tau,\ldots,t-1]$, of the same variables.}
\label{fig:endo_exo}
\end{figure}

\subsection{Preliminaries}

A causal directed acyclic graph (DAG) is defined over set $\nodes=\obs\cup\lat$ of observed $\obs$ and latent $\lat$ variables. 
We use an ancestral graph \citep{richardson2002ancestral} to model conditional independence relations among the observed variables. An ancestral graph that represents all conditional independence relations, and only those, is called a maximal ancestral graph (MAG). A MAG is a mixed graph with two types of edge-marks, head `{-}\textgreater' and tail `{--}'. Let $X$ and $Y$ be connected in a MAG. On the connecting edge, a tail edge-mark at $X$ indicates that $X$ is an ancestor of $Y$ in any DAG represented by this MAG, and a head edge-marks at $X$ indicates that $X$ is not an ancestor of $Y$.
An important property of MAGs is that for every DAG and a set of observed variables $\obs$, there exists a unique MAG that encodes all conditional independence relations among $\obs$ after marginalizing out $\lat$. It is therefore the learning goal of causal discovery in the presence of latent variables. In a MAG, a graphical criterion called m-separation determines conditional independence. An equivalence class of MAGs, where every member may encode only those conditional independence relations that are in the MAG, is called partial ancestral graph (PAG). A PAG that encodes all conditional independence relations encoded in the MAG is called a completed PAG. A PAG has three types of edge-marks, head, tail, and circle `o'. A circle edge-mark indicates position at which in the equivalence class there exist at least one MAG with a head edge-mark and one MAG with a tail edge-mark. Thus, a PAG representing no information is a fully-connected graph with only circle edge-marks. 

An $n$ dimensional SVAR process with a finite order $\tau$ can be defined as a structural equation model: $\forall t\in \mathbb{N}$ and $ \forall i\in\{1,\ldots,n\}$
\begin{equation}
    X_{i}^{t} = f_{i}(\nodes^{t}\setminus X_{i}^{t}, \nodes^{t-1}, \ldots, \nodes^{t-\tau}, \varepsilon_{i}^{t}),
\end{equation}
where $\{\varepsilon_{i}^{t}|t\in \mathbb{N}, i\in\{1,\ldots,n\}\}$ are mutually and serially independent noise terms. Each $\nodes^{s}$ consists of $n$ variables measured at time $s$. We assume this data-generating process is stationary. Corresponding to this process is a graphical model called dynamic-DAG \citep{malinsky2018causal}, where some variables are latent $\lat$ and the others $\obs$ are observed from the marginal stationary data-generating process. A MAG for this dynamic-DAG is called dynamic-MAG and an equivalence class for it is called dynamic-PAG.

\subsection{The TS-ICD Algorithm}

For a stationary, structural vector auto-regressive (SVAR) process, with order $\tau$, TS-ICD learns a dynamic-PAG $\pag$ \citep{malinsky2018causal} representing causal relations among observed variables. Let  $\obs^t=\{X_1^t, \ldots, X_n^t\}$ be the set of observed nodes at time-stamp $t$. Thus $\pag$ is a PAG over $\bigcup_{s=t-\tau}^t \obs^{s}$, and $P(\obs^{s}) = P(\obs^{s+r}), \forall r\in\mathbb{N}$ (stationarity). However, considering all the possible edges in this graph may be inefficient. Multiple possible edges may represent the exact same casual meaning. For example, an edge between two consecutive measurements of a variable $(X_i^{t-1}, X_i^{t})$ represents auto-correlation, and has the same meaning as the edge $(X_i^{t-\tau}, X_i^{t-\tau+1}), \forall \tau \geq 2$. From stationarity, faithfulness, and causal Markov assumptions, causal relations are consistent over time. Namely, in $\pag$ the edge between $X_i^{s}$ and $X_j^t$ is identical to the edge between $X_i^{k}$ and $X_j^l$ if $t-s=l-k$ \citep{entner2010causal}.

\begin{definition}[Homology]
A pair $(X_i^{s}, X_j^t)$ is homologous to pair $(X_i^{k}, X_j^l)$ if $t-s=l-k$. The set of all pairs homologous to $(X_i^{s}, X_j^t)$ are denoted $\homology(X_i^{s}, X_j^t)$.
\end{definition}

Let $\pag$ be a dynamic PAG defined over $\boldsymbol{V}$, a set of nodes measured at $t$, and $\boldsymbol{U}$, a set of nodes measured at $\tau$ earlier time-stamps $[t-\tau, \ldots, t-1]$. We call nodes in $\boldsymbol{V}$ endogenous, and nodes in $\boldsymbol{U}$ exogenous. We call edges among $\boldsymbol{V}$ contemporaneous, and edges connecting a node in $\boldsymbol{U}$ and a node in $\boldsymbol{V}$ temporal. Let the set of edges in $\pag$ be the union of temporal and contemporaneous edges. See \figref{fig:endo_exo} for an example. No connected pair in $\pag$ is homologous to another connected pair in $\pag$. Hence, this PAG consists of the minimal set of edges to be learned.

Next, we extend the notion of `autonomous sub-structure' proposed by \citet{yehezkel2009rai} to autonomous sub-graph of a PAG. For that, we use the following definitions.
For MAGs, a set called $\DSEPset{A}{B}$ contains the conditioning set by which $A$ and $B$ are m-separated (conditionally independent), if $A$ and $B$ are not adjacent and $A$ is not an ancestor of $B$ \citep[Theorem 6.2]{spirtes2000}.

\begin{definition}[D-SEP \citep{spirtes2000}]
If $\MAG$ is a MAG and $A\neq B$, $Z$ is in $\DSEPset{A}{B}$ if and only if $Z\neq A$ and there is a path between $A$ and $Z$ such that every node on the path is an ancestor of $A$ or $B$ and, except for the endpoints, is a collider.
\end{definition}

For PAGs, a super-set of $\DSEPset{A}{B}$, called $\PDSEPset{A}{B}$ \citep{spirtes2000}, includes all the nodes that cannot be ruled out from being in $\DSEPset{A}{B}$. A node $Z$ is in $\PDSEPset{A}{B}$ if and only if there is a PDS-Path from $Z$ to $A$.

\begin{definition}[PDS-path \citep{rohekar2021iterative}]\label{dff:pds-path}
    A possible-D-Sep-path (PDS-path) from $A$ to $Z$, with respect to $B$, in PAG $\pag$, denoted $\PDSPath{A}{B}{Z}$, is a path $\langle A,\ldots,Z\rangle$ such that $B$ is not on the path and for every sub-path $\langle U,V,W \rangle$ of $\PDSPath{A}{B}{Z}$, $V$ is a collider or $\{U, V, W\}$ forms a triangle.
\end{definition}

\citet{rohekar2021iterative}, defined a set smaller than $\PDSepop$, called $\ICDSep$, which complies with the following conditions. ICD-Sep conditions for $\mathbf{Z}\subset\ICDSep(A,B)$, given $r$, are
\begin{enumerate}
    \itemsep0em
    \item $|\mathbf{Z}|=r$,\label{ICD-size}
    \item $\forall Z\in\mathbf{Z}$, there exists a PDS-path $\PDSPath{A}{B}{Z}$ such that,\label{ICD-node-constraints}
    \begin{enumerate}
        \item $|\PDSPath{A}{B}{Z}| \leq r$ and\label{ICD-pathlen}
        \item every node on $\PDSPath{A}{B}{Z}$ is in $\mathbf{Z}$, \label{ICD-closed-set}
    \end{enumerate}
    \item $\forall Z\in\mathbf{Z}$, node $Z$ is a possible ancestor of $A$ or $B$.\label{ICD-possible-ancestor}
\end{enumerate}

We now extend the definition of autonomous sub-graph to cases where causal sufficiency cannot be assumed. That is, autonomy in PAG.

\begin{definition}[Autonomous sub-graph in a PAG]\label{dff:autonomous}
In a PAG $\pag$ over nodes $\obs$, a sub-graph $\pag_\mathrm{Au}$ over $\obs_\mathrm{Au}\subset\obs$ is called autonomous given a set $\obs_\mathrm{Ex}\subset\obs$ if $~\forall (A,B)\in\pag_\mathrm{Au}$,  $\ICDSep(A,B)$ and $\ICDSep(B,A)$ are subsets of $\obs_\mathrm{Ex}\cup\obs_\mathrm{Au}$. If $\obs_\mathrm{Ex}$ is empty, we say the sub-graph is (completely) autonomous.
\end{definition}

From \dffref{dff:autonomous} it follows that in a dynamic PAG, a sub-graph containing contemporaneous relations among endogenous nodes is an autonomous sub-graph given exogenous nodes. By identifying contemporaneous relations as autonomous, the overall number of CI tests required for learning the graph can be reduced. 

\citet{yehezkel2009rai} defined a recursive learning scheme called RAI which contains three main stages: 
\begin{enumerate}
    \itemsep0em
    \item Exogenous-endogenous refinement: remove edges connecting exogenous to endogenous nodes that are found independent conditioned on a set of size $r$,
    \item Endogenous refinement: remove of edges among endogenous nodes found independent conditioned on a set of size $r$,
    \item Autonomy identification: identify autonomous sub-graphs, and recursively call for each autonomous sub-graph using $r\leftarrow r+1$.
\end{enumerate}
By initially refining the edges between exogenous and endogenous nodes, the number of CI tests required in the second stage are reduced. This is because after stage 1, fewer nodes from the exogenous nodes are considered for a conditioning set in stage 2. In stage 3 they split into autonomous sub-graphs, by first identifying nodes having the lowest topological order calling them endogenous nodes, removing them temporarily, and identifying the remaining disconnected graphs calling them exogenous. Note that exogenous nodes are not required to be mutually independent.

In TS-ICD we start off with a fully-connected graph and refine edges connecting exogenous nodes to endogenous nodes before refining edges between the endogenous nodes. That is, we refine temporal relations before refining contemporaneous ones. However, the dynamic PAG can be split into endogenous and exogenous nodes in different manners. 

We derive an ordering for refining temporal relations, as entailed from different exogenous-endogenous splits. 
In an extreme case, $\obs^{t-\tau+1}\cup,\ldots,\cup\obs^t$ are set to be endogenous, and $\obs^{t-\tau}$ to be exogenous (\figref{fig:split_cases}, case \#1). In this extreme case, due to homology, edges between nodes $\obs^{t-\tau}$ and nodes in $\obs^{t-\tau+1}\cup,\ldots,\cup\obs^{t-1}$ (e.g., \figref{fig:split_cases} case \#1, edge $B$) are already included for endogenous nodes $\obs^{t-\tau+1}\cup,\ldots,\cup\obs^t$ (e.g., \figref{fig:split_cases} case \#1, edge $B'$). Thus, from the RAI learning paradigm, initially refining edges between $\tau$ apart nodes may reduce the number of CI tests required for refining edges between nodes that are less than $\tau$ apart (endogenous nodes). Now consider a different case where $\obs^{t-\tau+2}\cup,\ldots,\cup\obs^t$ are set to be endogenous, and $\obs^{t-\tau}\cup\obs^{t-\tau+1}$ to be exogenous (\figref{fig:split_cases} case \#2). Here, edges connecting nodes that are $\tau$ or $\tau-1$ apart are only temporal ones and an ordering between them cannot be determined. However, edges between $\tau-2$ apart nodes are both temporal and among endogenous nodes (\figref{fig:split_cases} edges $C$ and $C'$). Thus, initially refining of edges connecting $\tau$ and $\tau-1$ apart nodes may reduce the number of CI tests required for refining edges connecting $\tau-2$ apart nodes. From these cases, an ordering for refining temporal links is entailed: $[\tau, \tau-1, \ldots, 1]$, where longer-range (large interval) edges are refined first. This complies with the general observation that temporal relations are refined before refining contemporaneous ones.

\begin{figure}[t]
\centering
\includegraphics[width=1.0\columnwidth]{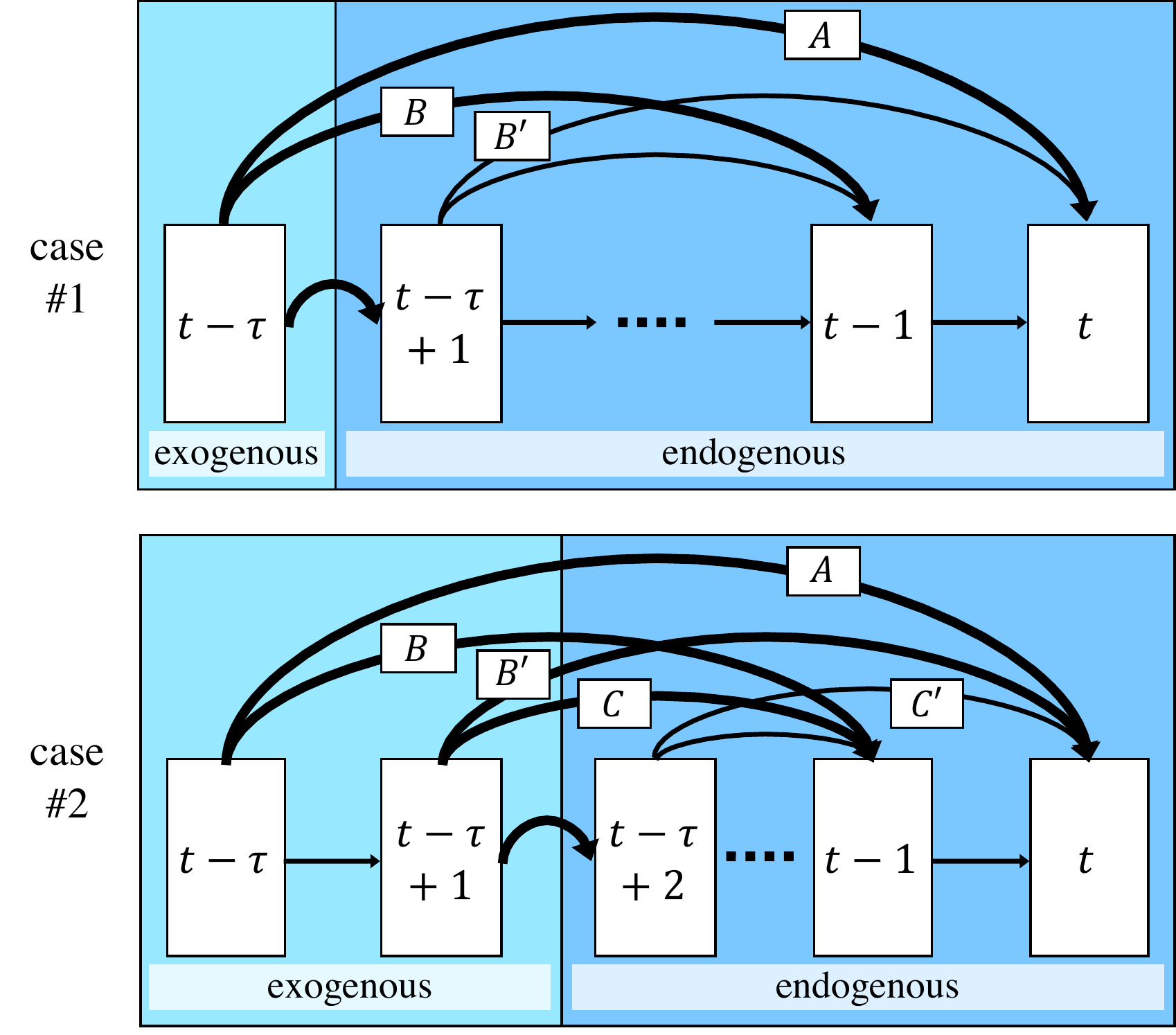}
\caption{An illustration for explaining how an ordering for refinement is derived. Edges marked $B$ and $B'$ are similar in the sense that they connect nodes with equal time-interval (homology) of $\tau-1$. Likewise, $C$ and $C'$ are similar having time-interval $\tau-2$. Case \#1 describe an extreme case in which only the earliest time-stamp is considered for exogenous nodes. In this case, edges $A$ and $B$ describe temporal relations but $B'$ is over endogenous nodes. Refinement of $A$ (interval length $\tau$) can reduce the number of CI tests required for refining $B'$ (interval length $\tau-1$). Thus, edge $A$ is refined before refining $B$. In case \#2, edges $A$, $B$, and $B'$ are all temporal relations and ordering between them cannot be determined; however, edge $B'$ (interval length $\tau-1$) is determined to be refined before $C'$ (interval length $\tau-2$). This leads to an ordering for refining temporal links, $[\tau, \tau-1,\tau-2,\ldots,1]$, based on the time interval between the connected nodes.}
\label{fig:split_cases}
\end{figure}

From these findings, we derive the TS-ICD algorithm, an adaptation of the ICD algorithm for time series data, using the RAI learning scheme to utilize SVAR assumption of homology and stationarity to reduce the number of CI tests. The pseudo-code is given in \algref{alg:Alg}. Initially, edge orientation for temporal relations are fixed such that they eliminate a causal effect from present to past (line 2). Similarly to ICD, \tsicd consists of a single iterative learning loop. However, differently from ICD, \tsicd uses an ordering for pairs of connected nodes, by which they are tested for independence. In each iteration, edges representing temporal relations are ordered in descending temporal distance (lines 5--6), and endogenous nodes are set to be $\obs^0$, present-time measurements (line 8). Next, RefineHomology function is called for refining temporal edges (line 7), and once concluded, it is called for refining contemporaneous edges (line 9). Finally, $r$ is increased by 1. The algorithm terminates if no conditioning sets that comply with the ICD-sep conditions are found.
The function RefineHomology tests conditional independence between pairs in a given ordered list ($\boldsymbol{E}$). Conditioning sets are constructed by calling a function PDSepRange \citep{rohekar2021iterative} that returns a list of conditioning sets that comply with the ICD-Sep conditions (line 16). If conditional independence is found between a pair $(X,Y)$, then all the edges connecting pairs in $\homology(X,Y)$ are removed (lines 21--22), and the conditioning set is recorded as a separating set for $(X,Y)$ (line 23). Finally, edges are oriented (line 25) following a set of deterministic orientation rules \citep{zhang2008completeness}, while keeping the initially set temporal orientations (line 2) fixed.

\begin{algorithm2e}[t!]
\label{alg:Alg}
\SetKwInput{KwInput}{Input}                
\SetKwInput{KwOutput}{Output}              
\SetKw{Break}{break}
\DontPrintSemicolon
\SetAlCapHSkip{0pt}
\caption{\tsicd}

  
  \KwInput{
    \\\quad $\obs$: a set of jointly observed variables
    \\\quad $\tau$: maximal time interval length
    \\\quad $\sindep$: a conditional independence oracle
    }
    
    
  \KwOutput{\\\quad $\pag$: a dynamic-PAG}

  \SetKwFunction{FMain}{Main}
  \SetKwFunction{FIter}{Iteration}
  \SetKwFunction{TestHomology}{RefineHomology}
  \SetKwFunction{FPDSepRange}{PDSepRange}

{
create $\tau$ fully-connected graphs, $\{\pag_{\mathrm{C}}^{t}|t\in[-\tau,\ldots,0]\}$, each over $\obs$ at time $t$ with `o{---}o' edges, representing contemporaneous relations
\;

initialize a dynamic PAG $\pag \assign \bigcup_{t=-\tau}^0(\pag_{\mathrm{C}}^{t})$, connect $X$o{---}\textgreater$Y$ for all $X\in\pag_{\mathrm{C}}^{t'}, Y\in\pag_{\mathrm{C}}^{t}, t'<t$\;

$done \assign \mathrm{False}$, $r \assign 0$}\;

\While{$(done=\mathrm{False})$ \& $(r \leq |\obs|-2)$}{
    assign temporal edges for refinement:
    $\boldsymbol{E}_{\mathrm{T}} \assign$ edges connecting $X\in\pag_{\mathrm{C}}^0$ and $Y\in\pag_{\mathrm{C}}^{t'}$ $\forall t'<0$\;
    
    sort $\boldsymbol{E}_{\mathrm{T}}$ in descending temporal distance\;
    
    $(\pag, done_{\mathrm{T}}) \assign \TestHomology(\boldsymbol{E}_{\mathrm{T}}, \pag, r)$\;
    
    assign contemporaneous edges for refinement:
    $\boldsymbol{E}_{\mathrm{C}} \assign$ edges between $X,Y\in\pag_{\mathrm{C}}^0$\; 
    
    $(\pag, done_{\mathrm{C}}) \assign \TestHomology(\boldsymbol{E}_{\mathrm{C}}, \pag, r)$\;
    
    
    $done \assign done_{\mathrm{T}}~\&~done_{\mathrm{C}}$\;
    $r \assign r+1$
    }
\KwRet $\pag$\;


  \SetKwProg{Fn}{Function}{:}{\KwRet}
  
  \Fn{\TestHomology{$\boldsymbol{E}$, $\pag$, $r$}}{
        $done \leftarrow \text{True}$\;
        
        \For{edge $(X, Y)$ ~in~ $\boldsymbol{E}$}{
            sets that comply with ICD-Sep conditions:
            $\{\mathbf{Z}_i\}_{i=1}^\ell \leftarrow$ \FPDSepRange($X$, $Y$, r, $\pag$)\;

            \If{$\ell > 0$}{
                  $done \leftarrow \text{False}$\; 

                \For{$i \assign 1$ \KwTo $\ell$}{
                    \If{$\sindep(X, Y| \mathbf{Z}_i)$}{
                        $\forall (V, U)\in\homology(X,Y,\pag)$ remove edge $(V, U)$ from $\pag$\;
                        
                        remove edge $(X, Y)$ from $\pag$\;
                        
                        record $\mathbf{Z}_i$ as a separating set for $(X, Y)$\;
                        
                        \Break\;
                    }
                }
            }
        }
        orient edges in $\pag$\;
        \KwRet $(\pag, done)$\;
  }
\end{algorithm2e}

\subsection{Bound on the Number of CI-Tests}

In a dynamic PAG, the number of possible contemporaneous edges is $\binom{n}{2}$, which is $O(n^{2})$ and number of possible temporal edges is $n^{2}$ for each lag (interval length) in $[1,\ldots, \tau]$, where $n$ is the number of nodes at a single time stamp. Thus, the total number of nodes is $(\tau+1)n$ and the total number of possible edges is bound by $(\tau+1)n^{2}$. The number of possible conditioning sets, having size $r$, is $\binom{(\tau+1) n-2}{r}$. A straightforward application of ICD to time-series results in an upper bound for the total number of CI-tests, $N_{ICD} (r)$, at iteration $r$ (conditioning set size),
\begin{equation}
    N_{ICD} (r) \leq (\tau+1) n^{2} \binom{(\tau+1) n -2}{r}.
\end{equation}
The proposed TS-ICD method uses autonomy identification to reduce the $(\tau+1) n$ term in the binomial coefficient that measures the number of possible conditioning sets.
Let $0<\rho< 1$ be the percentage of edges remaining after iteration $r$ (regardless of the time lag). Now recall that the edges with the longest lag are refined first, so the number of possible conditioning sets for edges with lag $\tau$ at iteration $r$ is the same as ICD, $\binom{(\tau+1) n-2}{r}$. This corresponds to edges of type A in \figref{fig:split_cases}. Next, edges with lag $\tau-1$ are tested. However there are fewer lag $\tau$ edges. So for testing lag $\tau-1$ edges, type B edges in \figref{fig:split_cases}, there are $\binom{\tau n + \rho n - 2}{r}$ possible conditioning sets. So an upper bound for the total number of CI-tests $N_{TSICD} (r)$ at iteration $r$ for TS-ICD is
\begin{equation}
    N_{TSICD} (r) \leq n^{2} \sum_{\ell=0}^{\tau} \binom{(\tau+1-\ell) n + \ell \rho n - 2}{r},
\end{equation}
where $\ell$ can be understood as $\tau-\mathrm{lag}$. It is evident that, fewer CI-tests are required by TS-ICD, compared to ICD, as $\rho$ decreases (sparser graphs) and as $\tau$ increases (considering longer-range temporal relations).

\subsection{Soundness and Completeness}

Given a perfect CI test, and given that the data is from a stationary SVAR process, the TS-ICD algorithm is sound and complete. It is sound in the sense that every independence relation and every causal relation that is entailed by the returned PAG are also present in the true underlying MAG. It is complete in the sense that every independence relations encoded by the true MAG and every causal relation that is entailed by these independence relations is also entailed by the returned PAG.

\begin{proposition}
If the distribution of variables in a stationary SVAR process is faithful to a dynamic-PAG, then TS-ICD recovers this dynamic-PAG, given a perfect CI-test. 
\end{proposition}
\begin{proof}
In the extreme case where all the nodes in the dynamic-PAG are considered endogenous (no exogenous nodes), the TS-ICD algorithm reduces to the ICD algorithm, which is sound and complete \citep{rohekar2021iterative}. Note that the proof for ICD considers an arbitrary ordering of CI testing as long as all subsets of all the nodes complying with ICD-Sep conditions are considered for conditioning.
In the general case, the split into exogenous and endogenous nodes results in an autonomous sub-graph composed of the exogenous nodes, and an autonomous sub-graph composed of the endogenous nodes conditioned on the exogenous nodes. By definition \citep{yehezkel2009rai}, autonomous sub-graphs (along with its conditioning set) include all required conditioning sets for identifying conditional independence between connected nodes in that autonomous sub-graph. In \dffref{dff:autonomous} we extended this definition to PAG such that an autonomous sub-graph (along with its conditioning set) includes sets complying with ICD-Sep conditions for its nodes. This ensures the soundness and completeness.
\end{proof}

\section{Empirical Evaluation}

We empirically examine the presented \tsicd algorithm, and compare it to 3 recently presented algorithms: SVAR-FCI \citep{malinsky2018causal}, SVAR-RFCI \citep{gerhardus2020high, colombo2012learning}, and LPCMCI ($k=4$) \citep{gerhardus2020high}. 
For implementation and empirical evaluation we used \emph{Causality Lab}: \url{github.com/IntelLabs/causality-lab}, and \emph{Tigramite}: \url{github.com/jakobrunge/tigramite} packages.

As LPCMCI was shown to achieve state-of-the-art results, we closely follow the numerical evaluation procedure used for evaluating it \citep{gerhardus2020high}. Specifically, we simulate data from two time series, one with continuous variables, and another with binary variables. In the first time-series, each variable is a linear combination of its parents and a normally distributed additive noise (a linear-Gaussian setting). The second time-series is a linear discretized into binary variables version of the first time-series. For more details on these data-generating models see \citet[Sections 4 and S9]{gerhardus2020high}. 

The parameters of each of the two time-series models is set as follows. We set $70\%$ of total variables as observed ($\obs$), and $30\%$ latent ($\lat$). We set the number of observed variable for a time-stamp to be 5. We set a temporal window of 4 time-stamps, $[t-3, t-2, t-1, t]$, (total of 20 observed variables). The length of the time series is set to 500; that is, 500 samples, each consisting of 5 variables. We sample 500 different models of these time-series, where the coefficients are sampled as follows. Auto-correlation coefficient (between previous and current measurements of a variable) is sampled uniformly from $[0.6, 0.9]$. Coefficients of additional 5 relations are sampled uniformly from $\pm [0.2,0.8]$. The percentage of contemporaneous relations is set to $30\%$ of the relations. 

For each algorithm, we use a CI-test based on partial correlation for the linear-Gaussian time-series, and G-square test for the discrete values time-series, with a $p$-value threshold $\alpha=0.01$. In addition, although the temporal time-window is 4 time-stamps, we use a window of 6 time-stamps $[t-5,\ldots, t]$ for the algorithms (a total of 30, 60, and 90 nodes in the learned graphs).

We evaluate three measures of performance for each of the tested algorithms: (1) F1-score for the skeleton accuracy, (2) causal accuracy, which is the percentage of correctly oriented edges, (3) required number CI tests. Since run-time may vary depending on implementation, and since CI-tests are generally the main time-consuming part, we report the number of CI tests instead of run-times. Results of the 500 instances of the time-series are summarized by the median and MAD (mean absolute deviation from median) values. Statistical significance is tested using the Wilcoxon signed-rank test at significance level $\alpha=0.05$.

\subsection{Ablation Study: Temporal before Contemporaneous?}

TS-ICD learns temporal causal relations before learning contemporaneous causal relations, and refines temporally longer-range relations before shorter-range relations. We examine two alternative orderings, different from TS-ICD's, by which pairs of nodes are tested for independence: (1) `swapped' and (2) random. In the `swapped' variant, contemporaneous relations are refined first, and the ordering for temporal relations is reversed. This ordering is expected to yield the worst results. In the `random' variant, the ordering for temporal and contemporaneous relations is shuffled. This is equivalent to a naive application of the ICD algorithm to time-series data. We then forced the TS-ICD algorithm to use these orderings and measured the performance. 

In \tabref{tab:ablation} we report median and MAD values of performance measures, summarizing 500 random instances of the binary-variables time-series model. The use of the `swapped' ordering resulted in performing the highest number of CI tests and the worst causal accuracy. TS-ICD achieved the highest causal accuracy and required the lowest number of CI tests. The superiority of TS-ICD in these cases was found to be statistically significant. We did not find any significant difference in the skeleton F1-score. To better examine the significance of the results, we provide in Appendix \ref{apx:ablation} scatter plots for the `random' and `swapped' variants in \figref{fig:random_scatter} and \figref{fig:swapped_scatter}, respectively. We found that TS-ICD and the tested variants required very few CI tests with conditioning sets higher than 1. This might be the reason why there is no significant difference in the skeleton accuracy.  Nevertheless, it is evident from the scatter plots that for most cases, TS-ICD requires fewer CI-tests and has higher causal accuracy compared to the tested variants.

\begin{table}
\centering
\small
\begin{tabular}{cccc}
    \toprule
    Variant & Skeleton F1 & Causal Acc. & \# CI-tests \\
    \midrule
    random  & 0.4744 & {0.2158} & {388}  \\
    & {\small(0.1403)} & {\small(0.1109)} & {\small(357)} \\
    \midrule
    swapped & 0.4744 & 0.2080 & 413  \\
     & {\small(0.1398)} &  {\small(0.1040)} &  {\small(356)} \\
    \midrule
    TS-ICD  & 0.4744 & \textbf{0.2321} & \textbf{371}  \\
    & {\small(0.1439)} & {\small(0.1186)} & {\small(363)} \\
    \bottomrule
\end{tabular}
\caption{Ablation study results. Median (and MAD) skeleton F1 score, causal accuracy, and number of CI tests are reported. In the `swapped' version, the contemporaneous edges were CI-tested before the temporal edges. In the `random' version, a random ordering for CI-testing was used. 
Results in bold are statistically significant at $\alpha=0.05$.}
\label{tab:ablation}
\end{table}

\begin{figure}[h!]
\centering
(a)\includegraphics[width=0.43\textwidth]{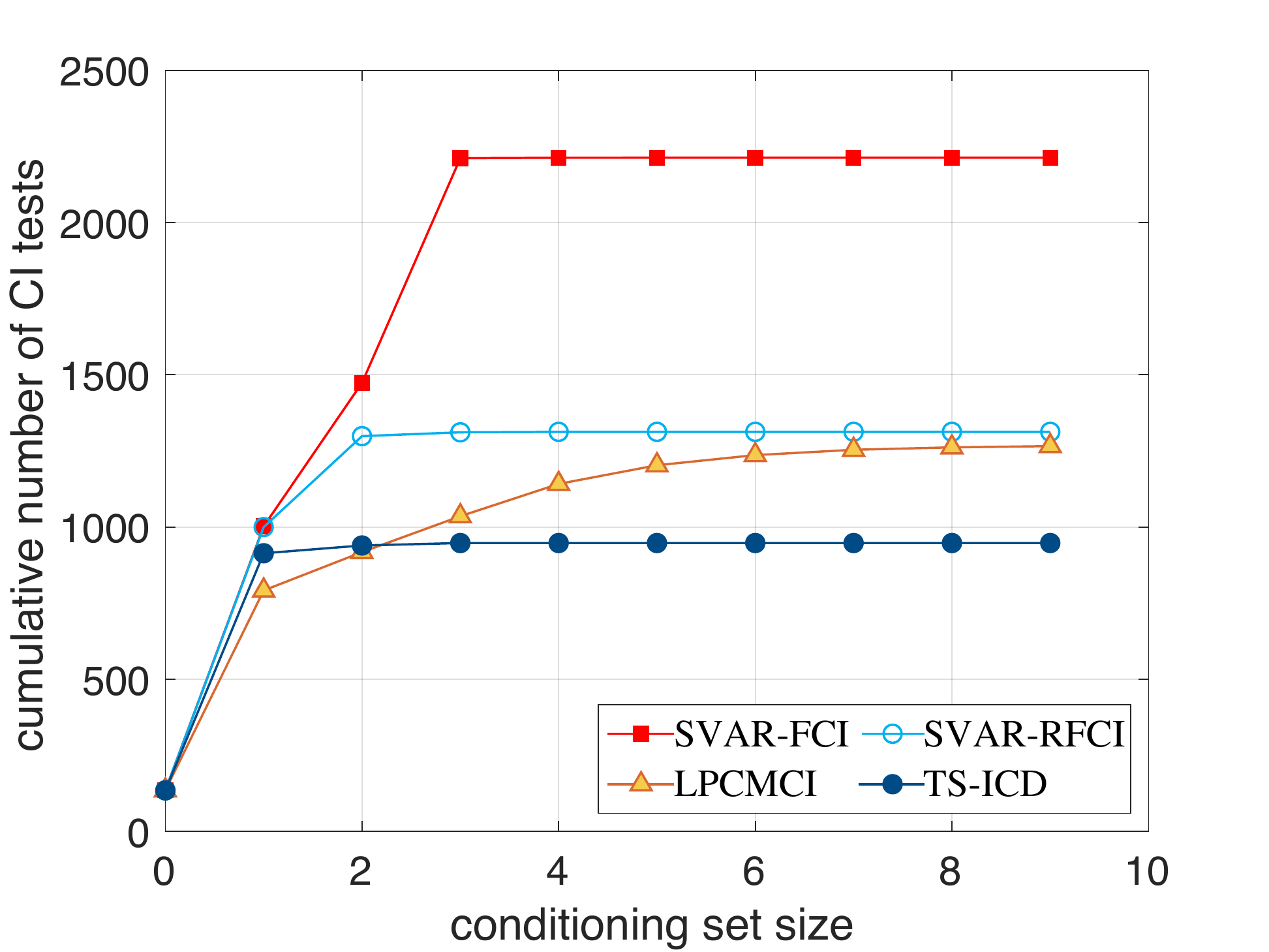}
(b)\includegraphics[width=0.43\textwidth]{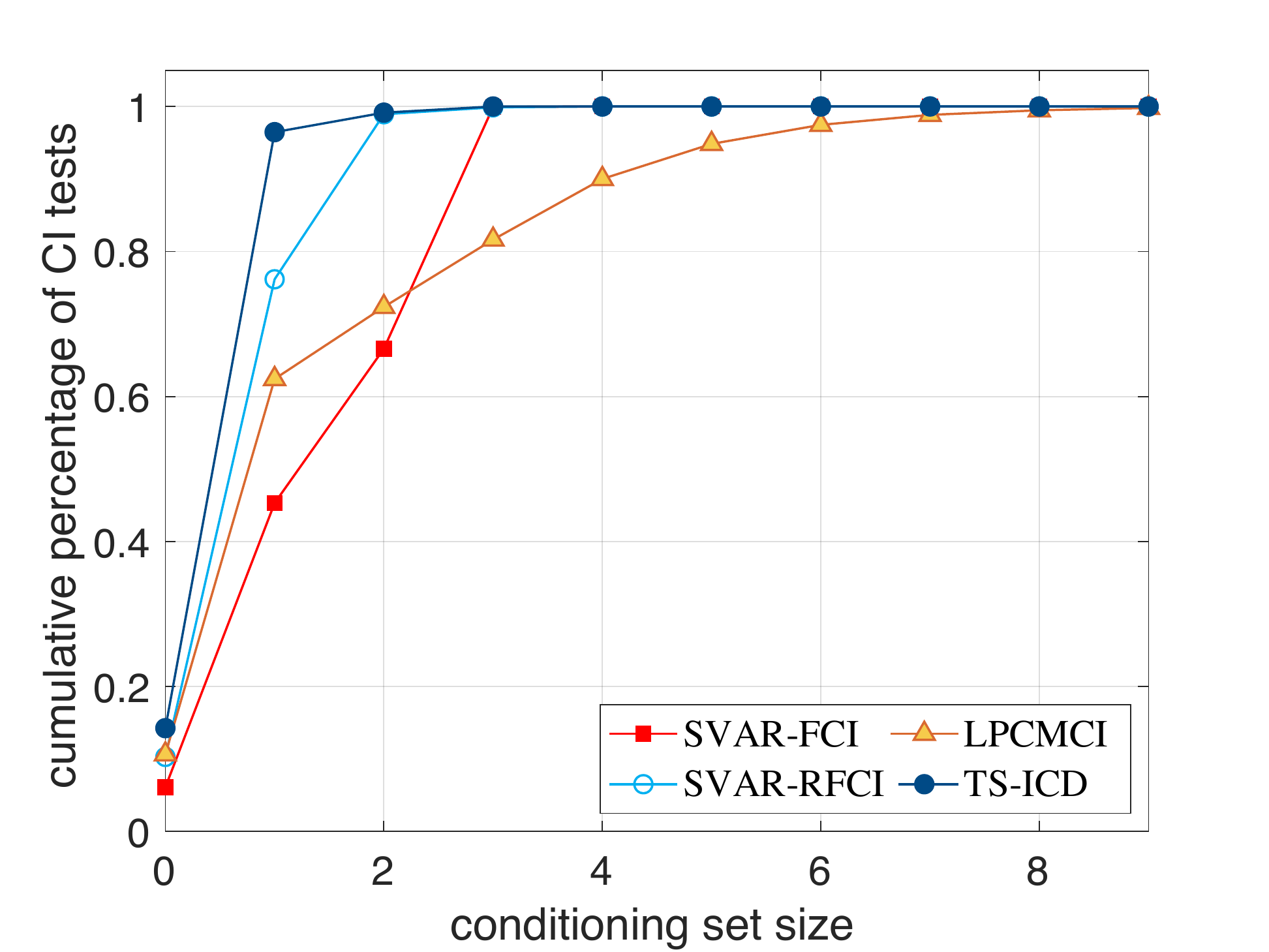}
\caption{Cumulative number (a) and percentage (b) of conditioning set sizes used by the CI tests. From (a) it is evident that TS-ICD requires the fewest number of CI tests in total. While, up to conditioning set size 2, both TS-ICD and LPCMCI require a similar number of CI tests, LPCMCI requires more than 300 additional CI tests, all with larger conditioning set sizes. From (b) it is evident that TS-ICD employs the highest percentage of conditioning set sizes up to 1, whereas SVAR-FCI employs the lowest percentage, and LPCMCI employs the lowest percentage of conditioning set sizes up to 3. This should affect accuracy of CI-test that suffer from the curse-of-dimensionality.}
\label{fig:cumulateNCI}
\end{figure}

\subsection{Comparison to Other Algorithms\label{sec:synthetic_experiment}}

\begin{table*}
\centering
\small
\begin{tabular}{l|ccc|ccc}
    \toprule
    & \multicolumn{3}{c|}{linear-Gaussian} & \multicolumn{3}{c}{discrete (binary)} \\
    Algorithm & Skeleton F1 & Causal Acc. & \# CI tests & Skeleton F1 & Causal Acc. & \# CI tests \\
    \midrule
    SVAR-FCI     & 0.3871            & 0.1987 & 1636 & \underline{0.3602}    & 0.2014 & 554\\
    & (0.1582) & (0.1421) & (1110) & (0.2438) & (0.2357) & (506) \\
    \midrule
    SVAR-RFCI    & 0.4118            & 0.2083 & \underline{1184} & \underline{0.3602}    & \underline{0.2083} & \underline{550}\\
    & (0.1488) & (0.1428) & (436) & (0.2507) & (0.2602) & (132) \\
    \midrule
    LPCMCI      & \textbf{0.5525}            & \textbf{0.2948} & 1192 & 0.2740                & 0.1443 & 1351\\
    & (0.1582) & (0.1908) & (368) & (0.2025) & (0.2219) & (1010) \\
    \midrule
    TS-ICD      & \underline{0.4838}    & \underline{0.2500} & \textbf{875} & \textbf{0.4744}       & \textbf{0.2321} & \textbf{371}\\
    & (0.0944) & (0.0987) & (282) & (0.1439) & (0.1186) & (363) \\
    \bottomrule
\end{tabular}
\caption{Results for learning a dynamic-PAG from a 500-samples long time-series. Median skeleton F1 score, causal accuracy, and number of CI tests. Highest accuracy values and lowest number of CI tests are marked bold, and second best are underlined.}
\label{tab:parcorr500}
\end{table*}

\begin{table}
\centering
\small
\begin{tabular}{l|cc|cc|cc}
    \toprule
    & \multicolumn{2}{c|}{Skeleton F1} & \multicolumn{2}{c|}{Causal Acc} & \multicolumn{2}{c}{\# CI tests}\\
    \# nodes & 60 & 90 & 60 & 90 & 60 & 90 \\
     \midrule
    LPCMCI & \textbf{0.44}	& {0.38} &	0.19 &	0.16 &	4907 &	11164\\
    TS-ICD & 0.41 &	{0.38} &	\textbf{0.23} &	\textbf{0.21} &	\textbf{3722} &	\textbf{8704}\\
    
    \bottomrule
\end{tabular}

\caption{Skeleton F1 score, causal accuracy, and number of CI tests (median over 500 graphs).} 
\label{tab:large_graphs}
\end{table}

In this section, the accuracy and computational complexity of the presented \tsicd algorithm is compared, using synthetic data, to the LPCMCI \citep{gerhardus2020high}, SVAR-FCI \citep{malinsky2018causal}, and SVAR-RFCI \citep{gerhardus2020high, colombo2012learning} algorithms. First, we evaluate the number of CI tests and their conditioning set sizes required by each algorithm for the linear-Gaussian time-series model. In \figref{fig:cumulateNCI} we provide the cumulative number and the cumulative percentage of CI-tests as a function of the conditioning set size. It is evident that \tsicd required the fewest number of CI tests and employed the highest percentage of small conditioning sets, up to size 1.

In \tabref{tab:parcorr500} we report performance measures for linear-Gaussian and binary-variable time-series. In the linear-Gaussian case LPCMCI achieved the highest skeleton F1-score and causal accuracy, and \tsicd achieved higher skeleton F1-score and causal accuracy than the other algorithms. In the binary-variable case \tsicd achieved the highest skeleton F1-score and causal accuracy, and LPCMI achieved the lowest. This is due to the sensitivity of G-square CI-test to the curse-of-dimensionality. For both cases TS-ICD required fewest CI tests. We further investigated the difference in performance between LPCMCI and TS-ICD in the linear-Gaussian case by increasing the number of nodes. We found that as graph sizes increase, the advantage of LPCMCI decreases, and for large graphs, having 90 nodes, TS-ICD outperforms LPCMCI (\tabref{tab:large_graphs}).

\section{Application to Real-World Data}

In this section we evaluate the application of TS-ICD to real-world data. To this end, we use the same data and analysis as by \citet{gerhardus2020high}, and compared to SVAR-FCI, SVAR-RFCI, and LPCMCI. The data is average daily discharges of rivers in the upper Danube basin. Measurements were made available by the Bavarian Environmental Agency at \url{https://www.gkd.bayern.de}. 
It includes measurements in three sites: Iller at Kempten $X$, Danube at Dillingen $Y$, and Isar at Lenggries, $Z$. See \figref{fig:river_sketch} for a schematic description of the rivers' flows and measurement points. For clarity, we recall the analysis and plausible causal relations made by \citet{gerhardus2020high}. The Iller discharges into the Danube upstream of Dillingen, where the water from Kempten reaches Dillingen in a day. Since the measurements are daily, the causal relation $X\longrightarrow Y$, is expected to be contemporaneous. The Isar discharges into the Danube downstream of Dillingen, and therefore there should be no causal relations between $Z$, and $X$ and $Y$. Nevertheless, there may be hidden confounders, such as rainfall, that may render these variables dependent. 

First, we called the algorithms using $p$-value threshold $\alpha=10^{-2}$. Resulting PAGs are given in \figref{fig:rivel-2}. 
SVAR-FCI learned the PAG $Y_{t-1}\longrightarrow Y_t \circleftrightarrow X_t \circleftrightarrow Z_t \circleftarrow Z_{t-1}$, and was omitted from the figure. Although it correctly identified the correlation between $X_t$ and $Y_t$, it did not identify the causal relation, and missed other relations that other algorithms commonly identified.
SVAR-RFCI and LPCMCI identified the expected causal relation $X_t\longrightarrow Y_t$, but LPCMCI also identified two causal relations that might not have a plausible explanation: $Z_t \longrightarrow Y_t$, and $X_{t-2}\longrightarrow Y_t$. TS-ICD returns a denser graph with bi-directed edges, which plausibly indicates the presence of weather-related hidden confounders. Next, we called the algorithms with a smaller threshold $\alpha=10^{-5}$, for which sparser graphs are expected. Resulting PAGs are given in \figref{fig:river-5}, where SVAR-FCI PAG was omitted as it was identical to the one learned with $\alpha=10^{-2}$. Both LPCMCI and TS-ICD omitted the relation between $Z_t$ and $Y_t$, and identified the expected relation $X_t\longrightarrow Y_t$, which SVAR-RFCI missed. The three algorithms still identified a confounder $Z \longleftrightarrow X$, but SVAR-RFCI omitted $Y_{t-2}\longleftrightarrow Y_t$ that is expected from weather conditions that may last more than a day. An important difference between TS-ICD, SVAR-RFCI and LPCMCI is the temporal relation between past $X$ and present $Y$ measurements. LPCMCI identified a relation $X_{t-2}\longrightarrow Y_t$, and no relation between $X_{t-1}$ and $Y_t$, whereas TS-ICD and SVAR-RFCI identified $X_{t-1} \longleftrightarrow Y$, which is more plausible. \tabref{tab:num_ci} summarizes the required number of CI tests, per conditioning set size. TS-ICD required significantly fewer CI-tests than all other algorithms, per conditioning set size, and in total.

\begin{figure*}
\centering
\includegraphics[width=0.8\textwidth]{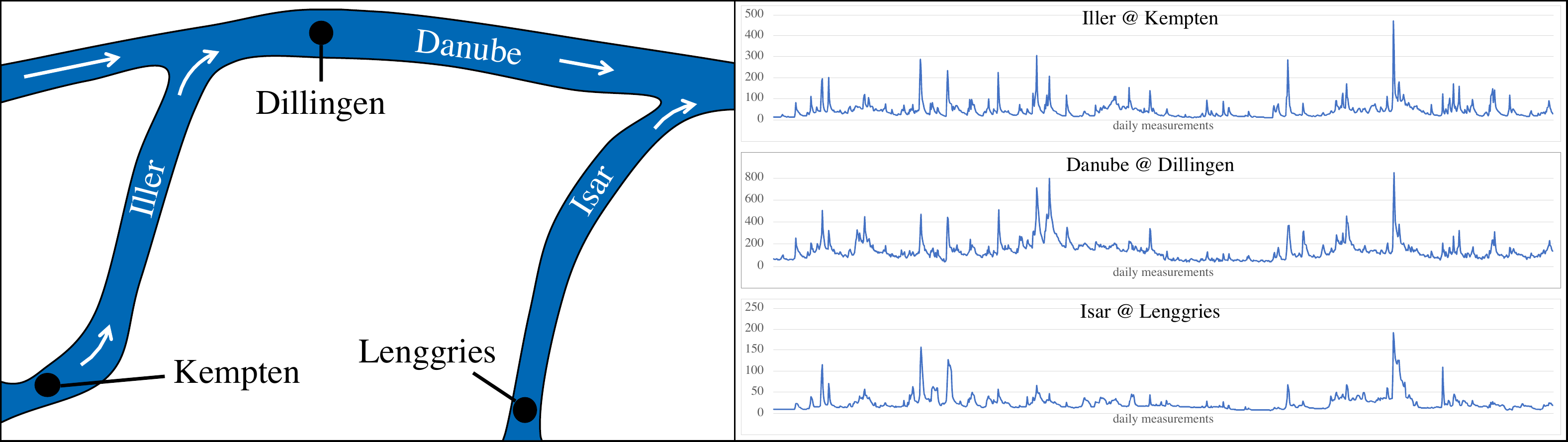}
\caption{A sketch of the flows of Iller, Isar, and Danube, and measurement points at Kempten ($X$), Dillingen ($Y$), and Lenggries ($Z$). Water from Kempten reaches Dillingen withing a day, thus a comtemporaneous relation is expected. In addition, measurements at Lenggries should not be effected by the flows of Iller and Danube. However, weather conditions, such as rainfall, may confound the measurements in all three sites.}
\label{fig:river_sketch}
\end{figure*}

\begin{figure}
\centering
(a)\includegraphics[width=0.4\textwidth]{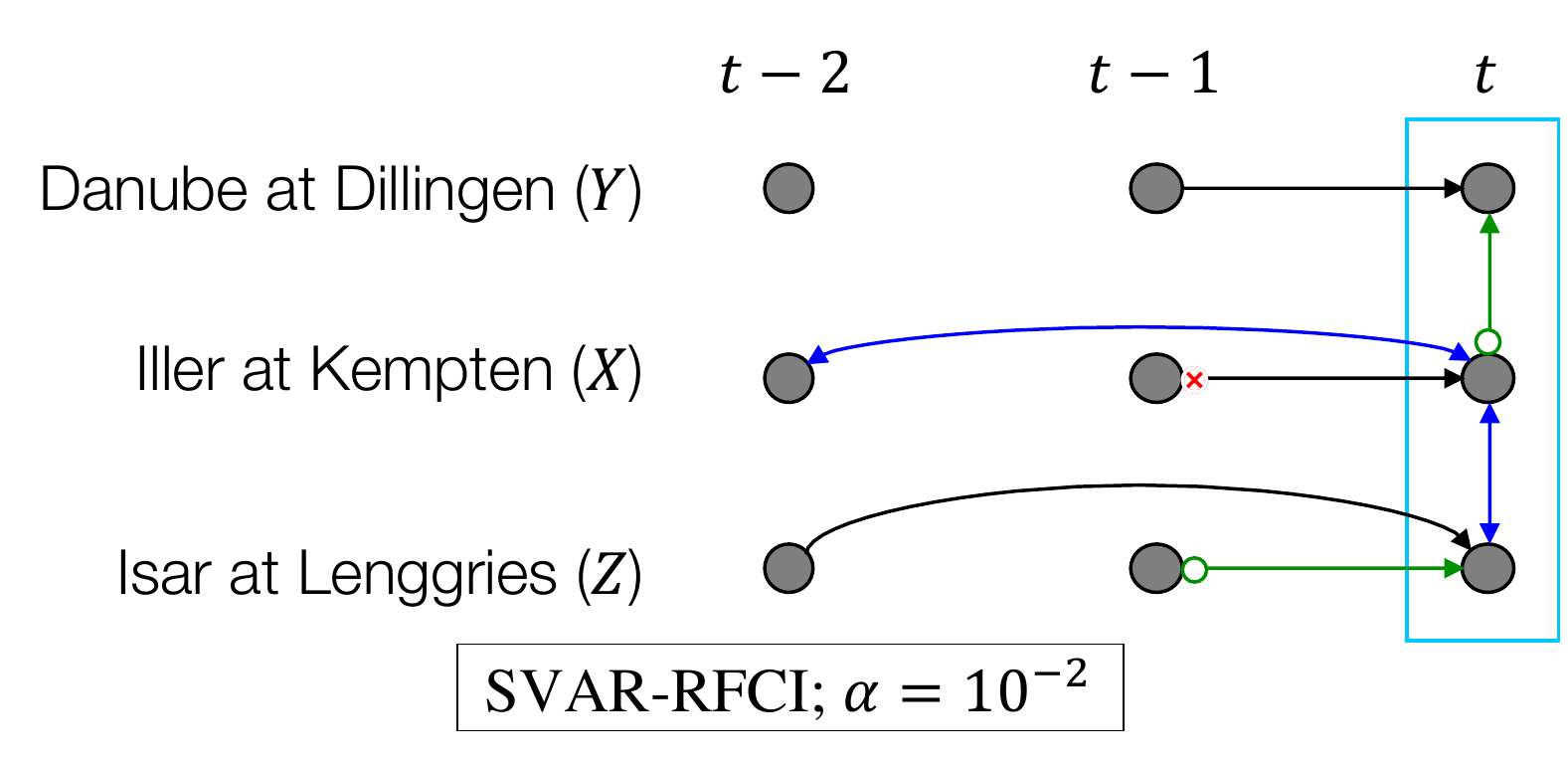}
(b)\includegraphics[width=0.4\textwidth]{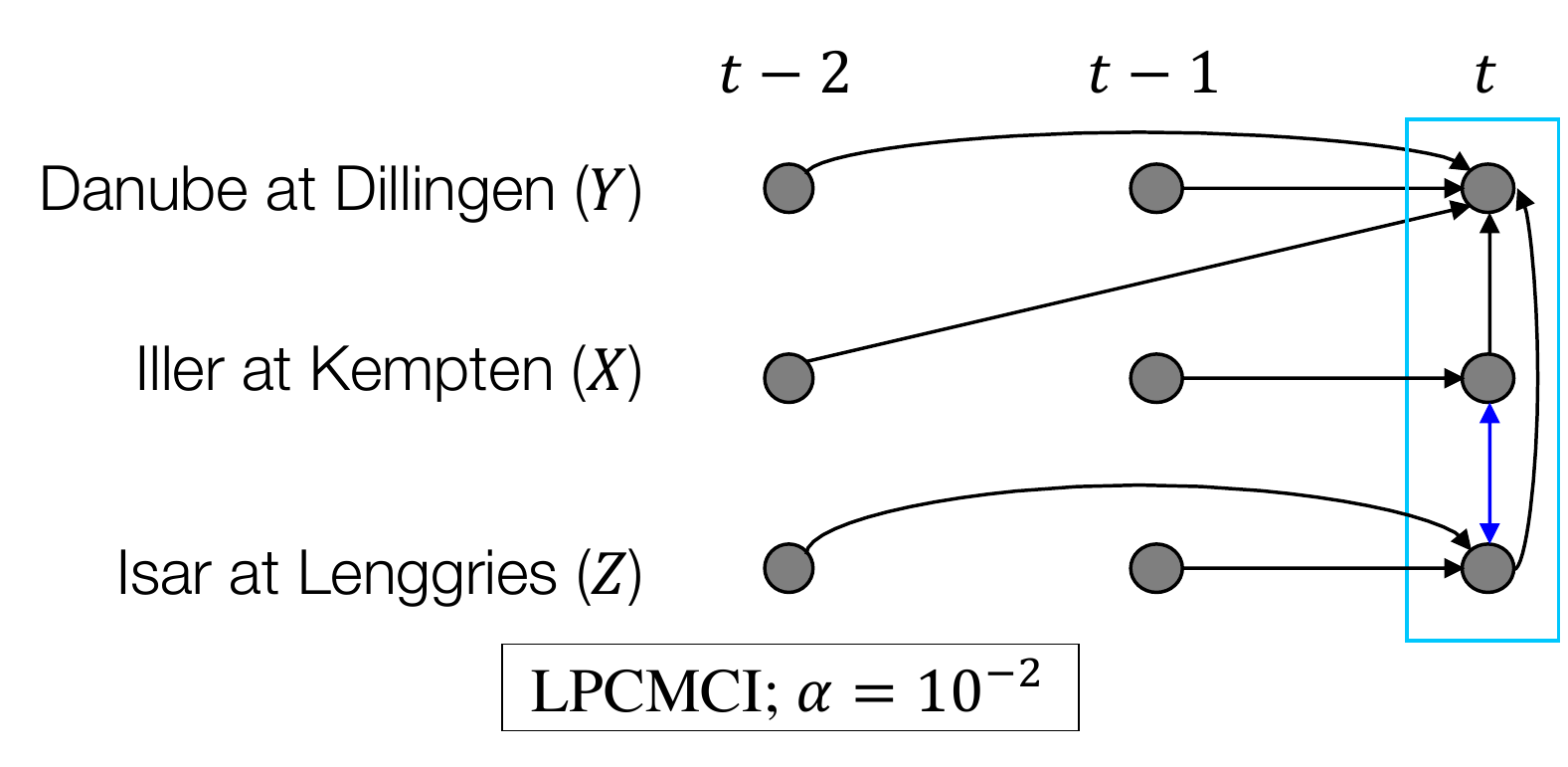}
(c)\includegraphics[width=0.4\textwidth]{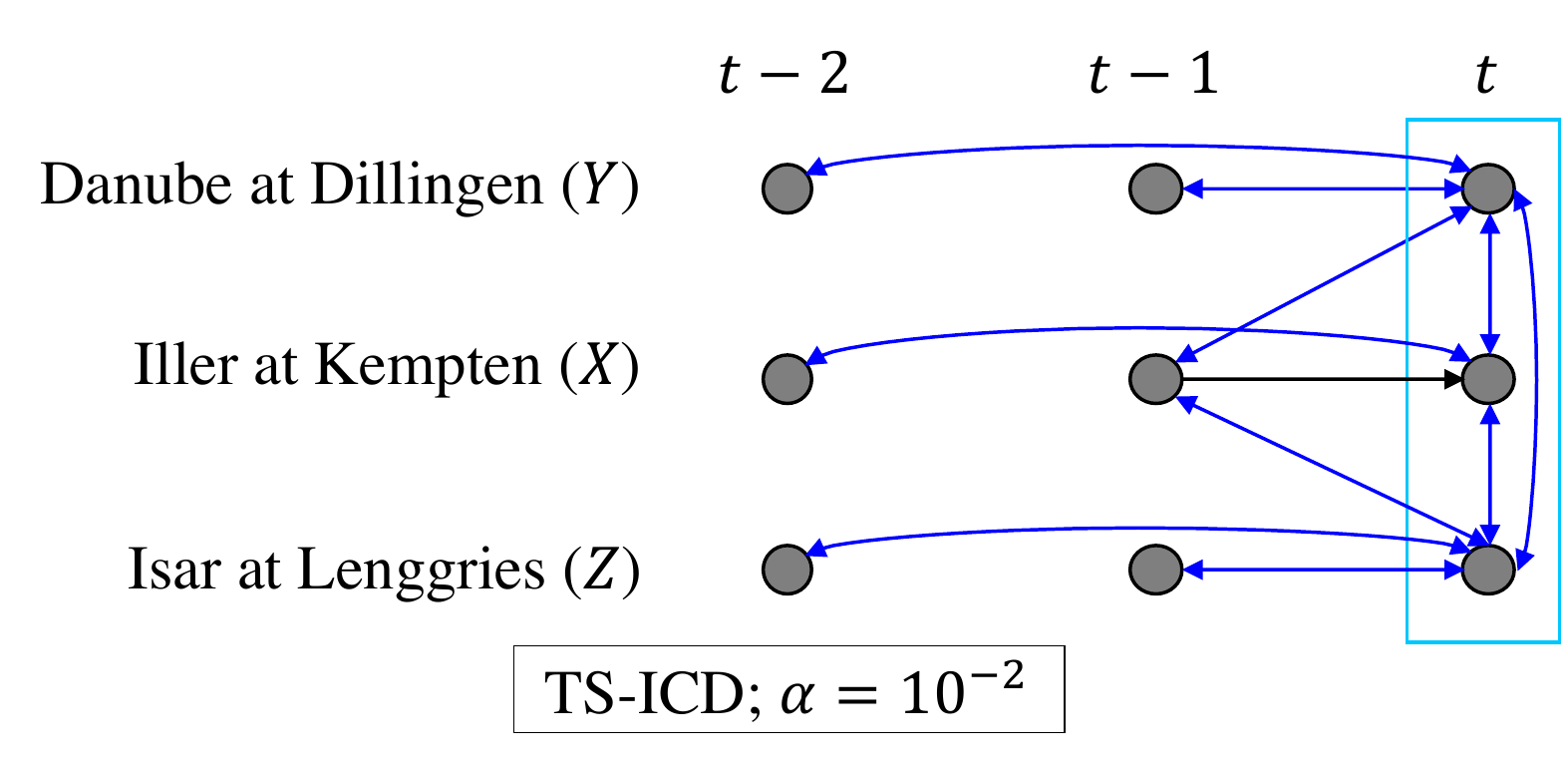}
\caption{PAGs returned at $\alpha=10^{-2}$ by (a) SVAR-RFCI, (b) LPCMCI, and (c) TS-ICD. Contemporaneous relations are bounded by a rectangle. Only one edge from each Homology is depicted. A red $\times$ indicates a conflict in determining the edge-mark.}
\label{fig:rivel-2}
\end{figure}

\begin{figure}
\centering
(a)\includegraphics[width=0.4\textwidth]{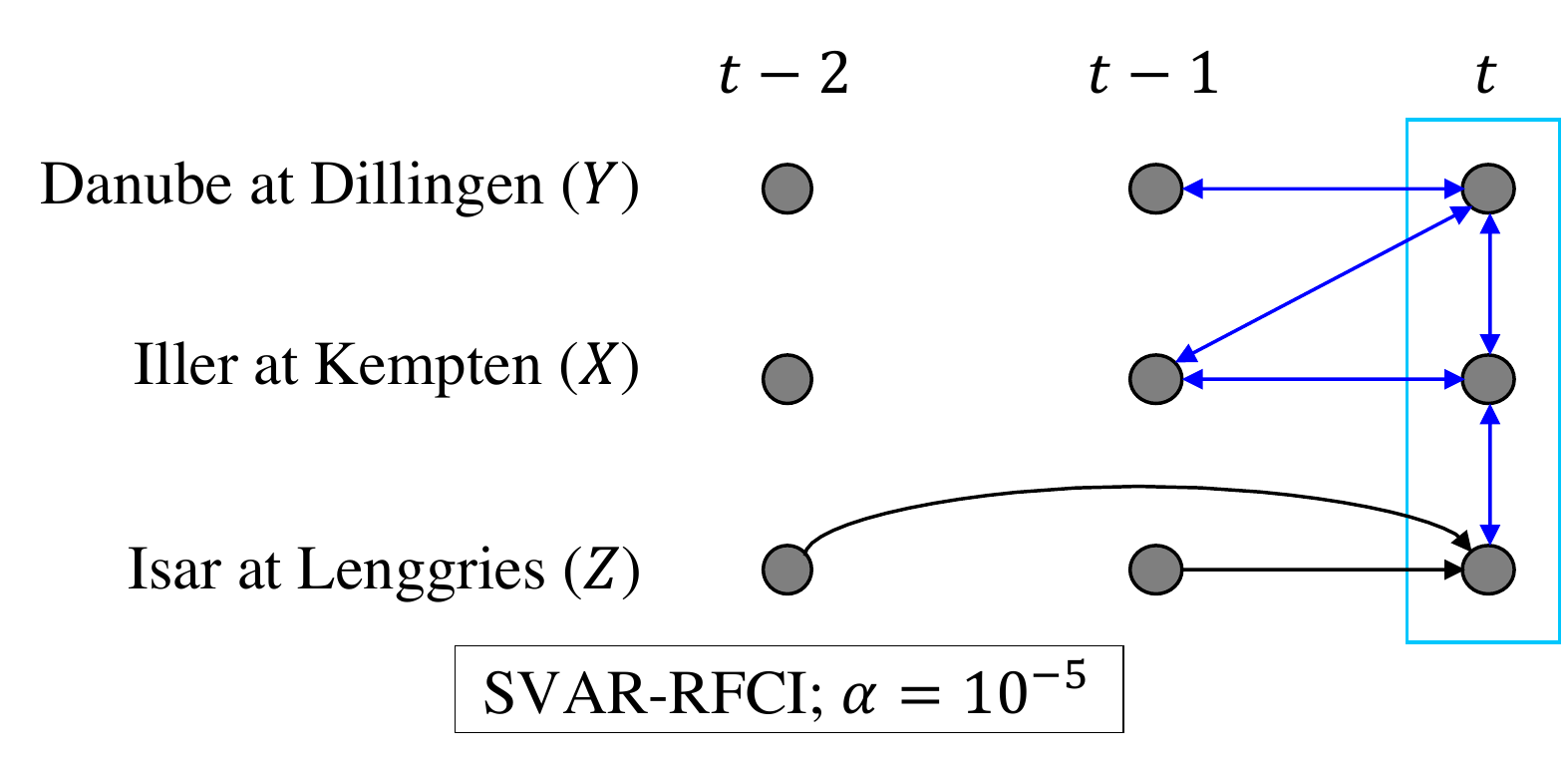}
(b)\includegraphics[width=0.4\textwidth]{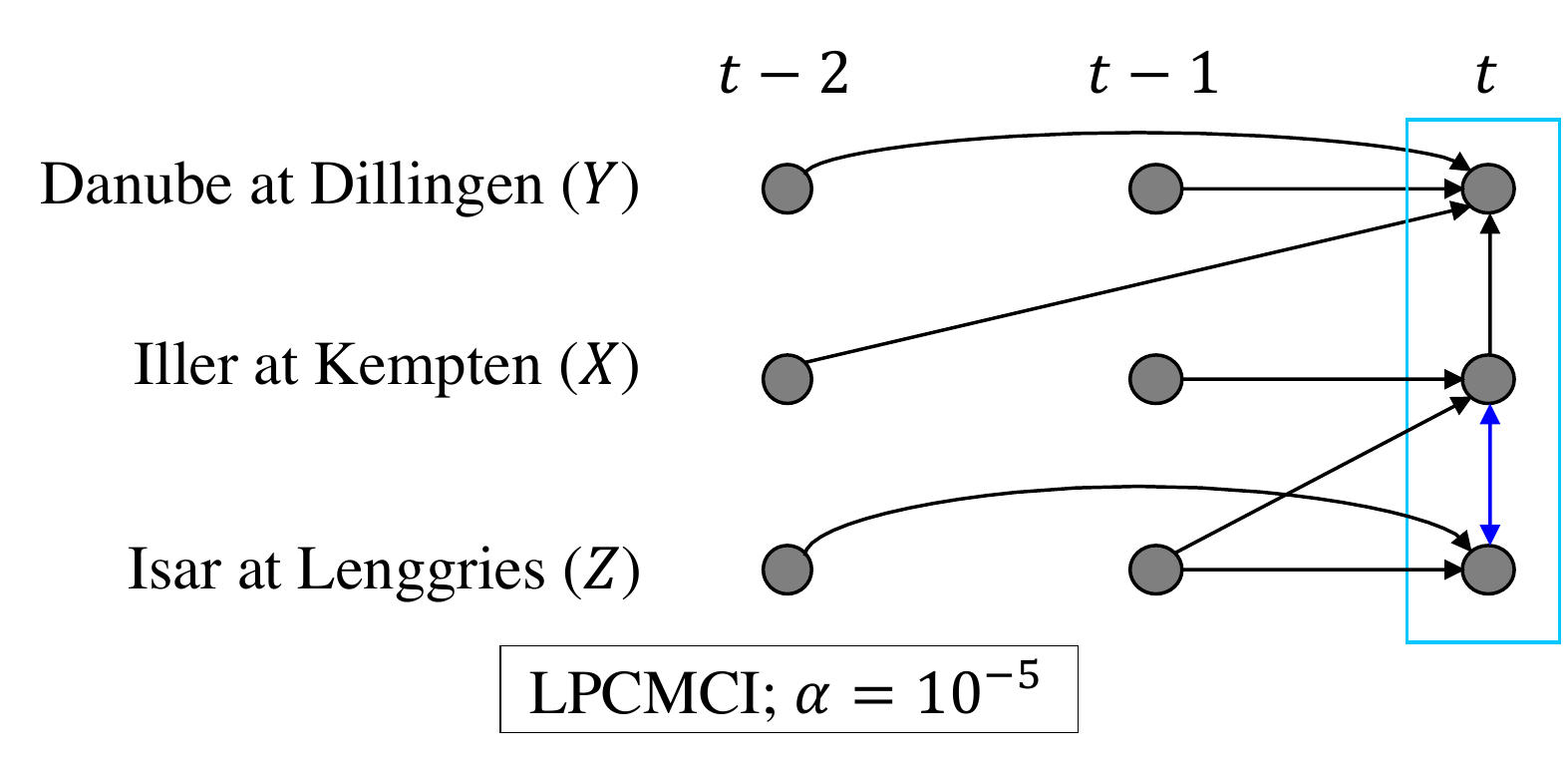}
(c)\includegraphics[width=0.4\textwidth]{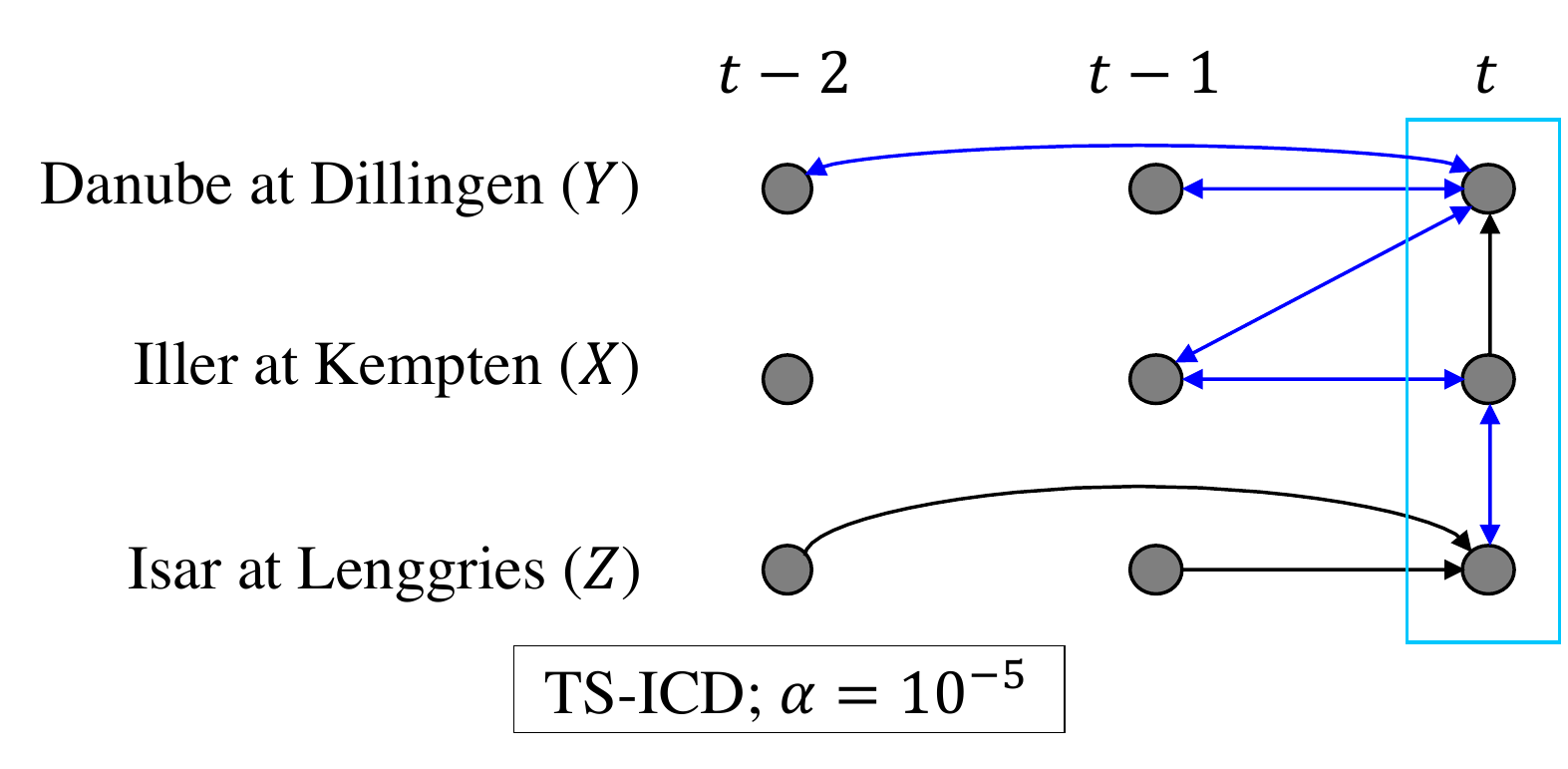}
\caption{PAGs returned at $\alpha=10^{-5}$ by (a) SVAR-RFCI, (b) LPCMCI, and (b) TS-ICD. Contemporaneous relations are bounded by a rectangle. Only one edge from each Homology is depicted.}
\label{fig:river-5}
\end{figure}

\begin{table*}
\centering
\small
\begin{tabular}{lrrrrrrrrc||rrrrrrrrc}
    \toprule
    & \multicolumn{8}{c}{Conditioning Set Size; $\alpha=10^{-2}$} & Total & \multicolumn{8}{c}{Conditioning Set Size; $\alpha=10^{-5}$} & Total\\
    Algorithm & 0 & 1 & 2 & 3 & 4 & 5 & 6 & 7 & & 0 & 1 & 2 & 3 & 4 & 5 & 6 & 7\\
     \midrule
     SVAR-FCI  & 21 & 115 & 177 & 192 & 144 & 84 & 28 & 4 & 765 & 21 & 103 & 177 & 183 & 140 & 84 & 28 & 4 & 740\\
     SVAR-RFCI & 21 & 115 & 86 & 23 & 4 & 0 & 0 & 0 & 249 & 21 & 103 & 48 & 11 & 0 & 0 & 0 & 0 & 183\\
    LPCMCI & 21 & 107 & 71 & 53 & 23 & 17 & 4 & 0 & 296 & 21 & 91 & 22 & 15 & 6 & 4 & 1 & 0 & 160\\
    TS-ICD & 21 & \textbf{97} & \textbf{2} & \textbf{0} & \textbf{0} & {0} & {0} & 0 & \textbf{135} & 21 & \textbf{78} & \textbf{2} & \textbf{0} & {0} & {0} & {0} & 0 & \textbf{107}\\
    
    \bottomrule
\end{tabular}
\caption{Number of CI tests counted for each conditioning set size.}
\label{tab:num_ci}
\end{table*}

\section{Conclusions}

In this paper we presented the \tsicd algorithm for efficiently learning from time-series data in the presence of latent confounders. By splitting the nodes into exogenous and endogenous, \tsicd creates an ordering according to which pairs of nodes are tested for conditional independence, which leads to a reduction in the number of CI tests. Our empirical evaluation validates reduction in the overall number of CI-tests, and demonstrates an increase in the accuracy of learned causal graphs, compared to a naive application of the ICD algorithm. The empirical evaluation also demonstrates that \tsicd mitigates the curse of dimensionality of CI-tests better than other algorithms, and returns more plausible causal graphs for real-world data.

\newpage
\bibliography{TSICD}

\begin{thebibliography}{25}
\providecommand{\natexlab}[1]{#1}
\providecommand{\url}[1]{\texttt{#1}}
\expandafter\ifx\csname urlstyle\endcsname\relax
  \providecommand{\doi}[1]{doi: #1}\else
  \providecommand{\doi}{doi: \begingroup \urlstyle{rm}\Url}\fi

\bibitem[Claassen et~al.(2013)Claassen, Mooij, and
  Heskes]{claassen2013learning}
Claassen, T., Mooij, J.~M., and Heskes, T.
\newblock Learning sparse causal models is not {NP}-hard.
\newblock In \emph{Uncertainty in Artificial Intelligence}, pp.\  172.
  Citeseer, 2013.

\bibitem[Colombo et~al.(2012)Colombo, Maathuis, Kalisch, and
  Richardson]{colombo2012learning}
Colombo, D., Maathuis, M.~H., Kalisch, M., and Richardson, T.~S.
\newblock Learning high-dimensional directed acyclic graphs with latent and
  selection variables.
\newblock \emph{The Annals of Statistics}, pp.\  294--321, 2012.

\bibitem[Entner \& Hoyer(2010)Entner and Hoyer]{entner2010causal}
Entner, D. and Hoyer, P.~O.
\newblock On causal discovery from time series data using fci.
\newblock \emph{Probabilistic graphical models}, pp.\  121--128, 2010.

\bibitem[Gerhardus \& Runge(2020)Gerhardus and Runge]{gerhardus2020high}
Gerhardus, A. and Runge, J.
\newblock High-recall causal discovery for autocorrelated time series with
  latent confounders.
\newblock \emph{Advances in Neural Information Processing Systems},
  33:\penalty0 12615--12625, 2020.

\bibitem[Jaber et~al.(2018)Jaber, Zhang, and Bareinboim]{jaber2018causal}
Jaber, A., Zhang, J., and Bareinboim, E.
\newblock Causal identification under markov equivalence.
\newblock In \emph{34th Conference on Uncertainty in Artificial Intelligence},
  pp.\  978--987. Association for Uncertainty in Artificial Intelligence
  (AUAI), 2018.

\bibitem[Jaber et~al.(2019)Jaber, Zhang, and Bareinboim]{jaber2019causal}
Jaber, A., Zhang, J., and Bareinboim, E.
\newblock Causal identification under markov equivalence: Completeness results.
\newblock In \emph{International Conference on Machine Learning}, pp.\
  2981--2989, 2019.

\bibitem[Malinsky \& Spirtes(2018)Malinsky and Spirtes]{malinsky2018causal}
Malinsky, D. and Spirtes, P.
\newblock Causal structure learning from multivariate time series in settings
  with unmeasured confounding.
\newblock In \emph{Proceedings of 2018 ACM SIGKDD workshop on causal
  discovery}, pp.\  23--47. PMLR, 2018.

\bibitem[Nisimov et~al.(2021)Nisimov, Gurwicz, Rohekar, and
  Novik]{nisimov2021improving}
Nisimov, S., Gurwicz, Y., Rohekar, R.~Y., and Novik, G.
\newblock Improving efficiency and accuracy of causal discovery using a
  hierarchical wrapper.
\newblock In \emph{Uncertainty in Artificial Intelligence (UAI 2021), the 4th
  Workshop on Tractable Probabilistic Modeling}, 2021.

\bibitem[Nisimov et~al.(2022)Nisimov, Rohekar, Gurwicz, Koren, and
  Novik]{nisimov2022clear}
Nisimov, S., Rohekar, R.~Y., Gurwicz, Y., Koren, G., and Novik, G.
\newblock {CLEAR}: Causal explanations from attention in neural recommenders.
\newblock \emph{arXiv preprint arXiv:2210.10621}, 2022.

\bibitem[Pearl(2009)]{pearl2009causality}
Pearl, J.
\newblock \emph{Causality: Models, Reasoning, and Inference}.
\newblock Cambridge university press, second edition, 2009.

\bibitem[Pearl(2010)]{pearl2010introduction}
Pearl, J.
\newblock An introduction to causal inference.
\newblock \emph{The international journal of biostatistics}, 6\penalty0 (2),
  2010.

\bibitem[Pearl \& Verma(1991)Pearl and Verma]{pearl1991theory}
Pearl, J. and Verma, T.
\newblock A theory of inferred causation.
\newblock In \emph{International Conference on Principles of Knowledge
  Representation and Reasoning}, pp.\  441--452, 1991.

\bibitem[Peters et~al.(2017)Peters, Janzing, and Sch\"olkopf]{peters2017}
Peters, J., Janzing, D., and Sch\"olkopf, B.
\newblock \emph{Elements of Causal Inference: Foundations and Learning
  Algorithms}.
\newblock MIT Press, Cambridge, MA, USA, 2017.

\bibitem[Richardson \& Spirtes(2002)Richardson and
  Spirtes]{richardson2002ancestral}
Richardson, T. and Spirtes, P.
\newblock Ancestral graph markov models.
\newblock \emph{The Annals of Statistics}, 30\penalty0 (4):\penalty0 962--1030,
  2002.

\bibitem[Rohekar et~al.(2018{\natexlab{a}})Rohekar, Gurwicz, Nisimov, Koren,
  and Novik]{rohekar2018bayesian}
Rohekar, R.~Y., Gurwicz, Y., Nisimov, S., Koren, G., and Novik, G.
\newblock {B}ayesian structure learning by recursive bootstrap.
\newblock In \emph{Advances in Neural Information Processing Systems
  (NeurIPS)}, 2018{\natexlab{a}}.

\bibitem[Rohekar et~al.(2018{\natexlab{b}})Rohekar, Nisimov, Gurwicz, Koren,
  and Novik]{rohekar2018constructing}
Rohekar, R.~Y., Nisimov, S., Gurwicz, Y., Koren, G., and Novik, G.
\newblock Constructing deep neural networks by {B}ayesian network structure
  learning.
\newblock In \emph{Advances in Neural Information Processing Systems
  (NeurIPS)}, 2018{\natexlab{b}}.

\bibitem[Rohekar et~al.(2019)Rohekar, Gurwicz, Nisimov, and
  Novik]{rohekar2019modeling}
Rohekar, R.~Y., Gurwicz, Y., Nisimov, S., and Novik, G.
\newblock Modeling uncertainty by learning a hierarchy of deep neural
  connections.
\newblock \emph{Advances in neural information processing systems}, 32, 2019.

\bibitem[Rohekar et~al.(2021)Rohekar, Nisimov, Gurwicz, and
  Novik]{rohekar2021iterative}
Rohekar, R.~Y., Nisimov, S., Gurwicz, Y., and Novik, G.
\newblock Iterative causal discovery in the possible presence of latent
  confounders and selection bias.
\newblock \emph{Advances in Neural Information Processing Systems},
  34:\penalty0 2454--2465, 2021.

\bibitem[Spirtes(2010)]{spirtes2010introduction}
Spirtes, P.
\newblock Introduction to causal inference.
\newblock \emph{Journal of Machine Learning Research}, 11\penalty0
  (May):\penalty0 1643--1662, 2010.

\bibitem[Spirtes et~al.(1999)Spirtes, Meek, and
  Richardson]{spirtes1999algorithm}
Spirtes, P., Meek, C., and Richardson, T.
\newblock An algorithm for causal inference in the presence of latent variables
  and selection bias.
\newblock \emph{Computation, causation, and discovery}, 21:\penalty0 1--252,
  1999.

\bibitem[Spirtes et~al.(2000)Spirtes, Glymour, and Scheines]{spirtes2000}
Spirtes, P., Glymour, C., and Scheines, R.
\newblock \emph{Causation, Prediction and Search}.
\newblock {MIT} Press, 2nd edition, 2000.

\bibitem[Sugahara et~al.(2022)Sugahara, Kishida, Kato, and
  Ueno]{pmlr-v186-sugahara22a}
Sugahara, S., Kishida, W., Kato, K., and Ueno, M.
\newblock Recursive autonomy identification-based learning of augmented naive
  bayes classifiers.
\newblock In \emph{Proceedings of The 11th International Conference on
  Probabilistic Graphical Models}, volume 186 of \emph{Proceedings of Machine
  Learning Research}, pp.\  265--276. PMLR, 2022.

\bibitem[Yehezkel \& Lerner(2009)Yehezkel and Lerner]{yehezkel2009rai}
Yehezkel, R. and Lerner, B.
\newblock {B}ayesian network structure learning by recursive autonomy
  identification.
\newblock \emph{Journal of Machine Learning Research (JMLR)}, 10\penalty0
  (Jul):\penalty0 1527--1570, 2009.

\bibitem[Zhang(2008{\natexlab{a}})]{zhang2008causal}
Zhang, J.
\newblock Causal reasoning with ancestral graphs.
\newblock \emph{Journal of Machine Learning Research (JMLR)}, 9:\penalty0
  1437--1474, 2008{\natexlab{a}}.

\bibitem[Zhang(2008{\natexlab{b}})]{zhang2008completeness}
Zhang, J.
\newblock On the completeness of orientation rules for causal discovery in the
  presence of latent confounders and selection bias.
\newblock \emph{Artificial Intelligence}, 172\penalty0 (16-17):\penalty0
  1873--1896, 2008{\natexlab{b}}.

\end{thebibliography}
\bibliographystyle{icml2023}

\newpage
\appendix
\section{Ablation: Temporal before Contemporaneous?\label{apx:ablation}}

In this section, scatter plots are provided (\figref{fig:random_scatter}, \figref{fig:swapped_scatter}) for evaluating the performance as a function of the refinement ordering.

\begin{figure}[b!]
\centering
(a)\includegraphics[width=0.90\columnwidth]{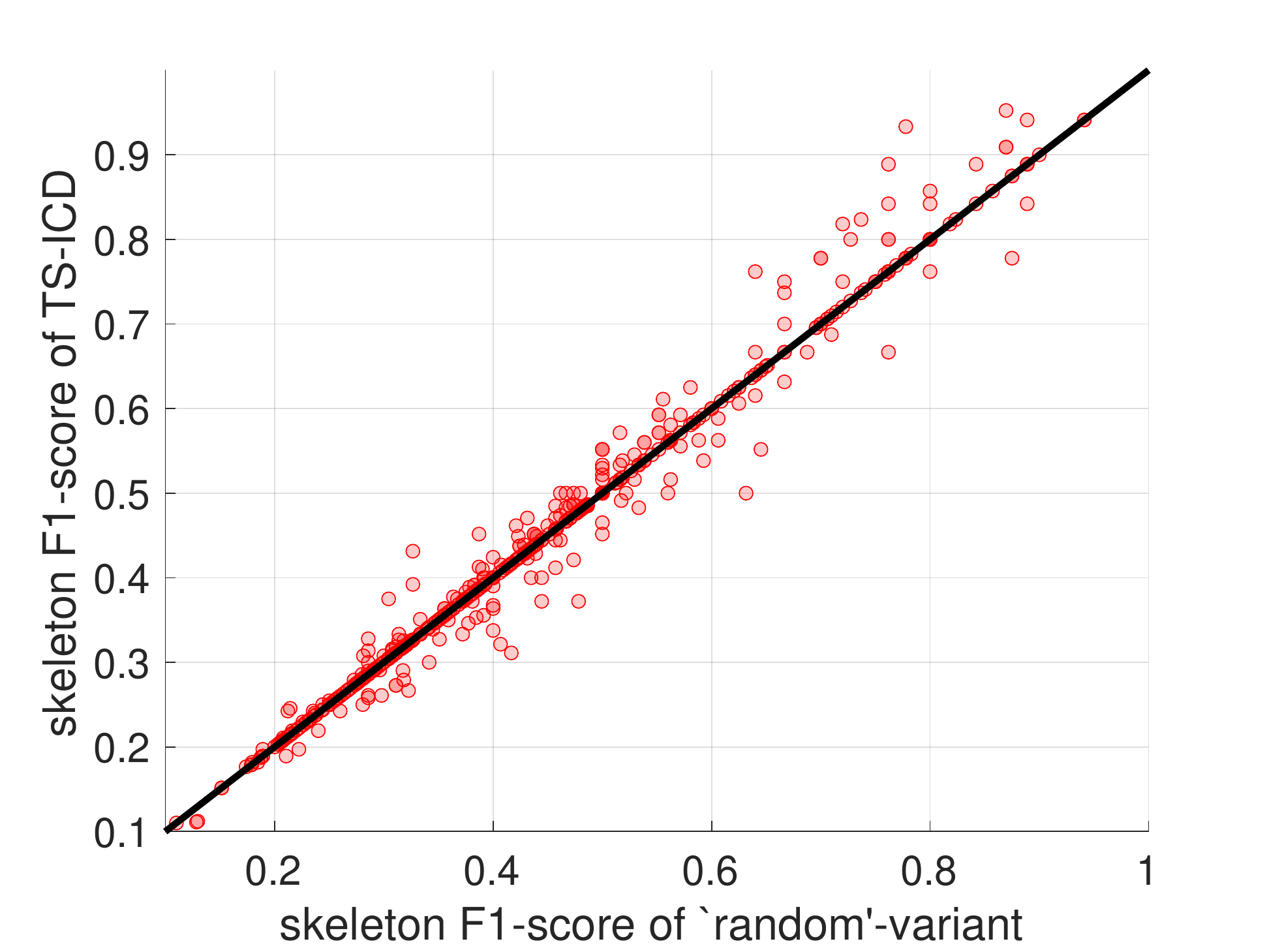}\\
(b)\includegraphics[width=0.90\columnwidth]{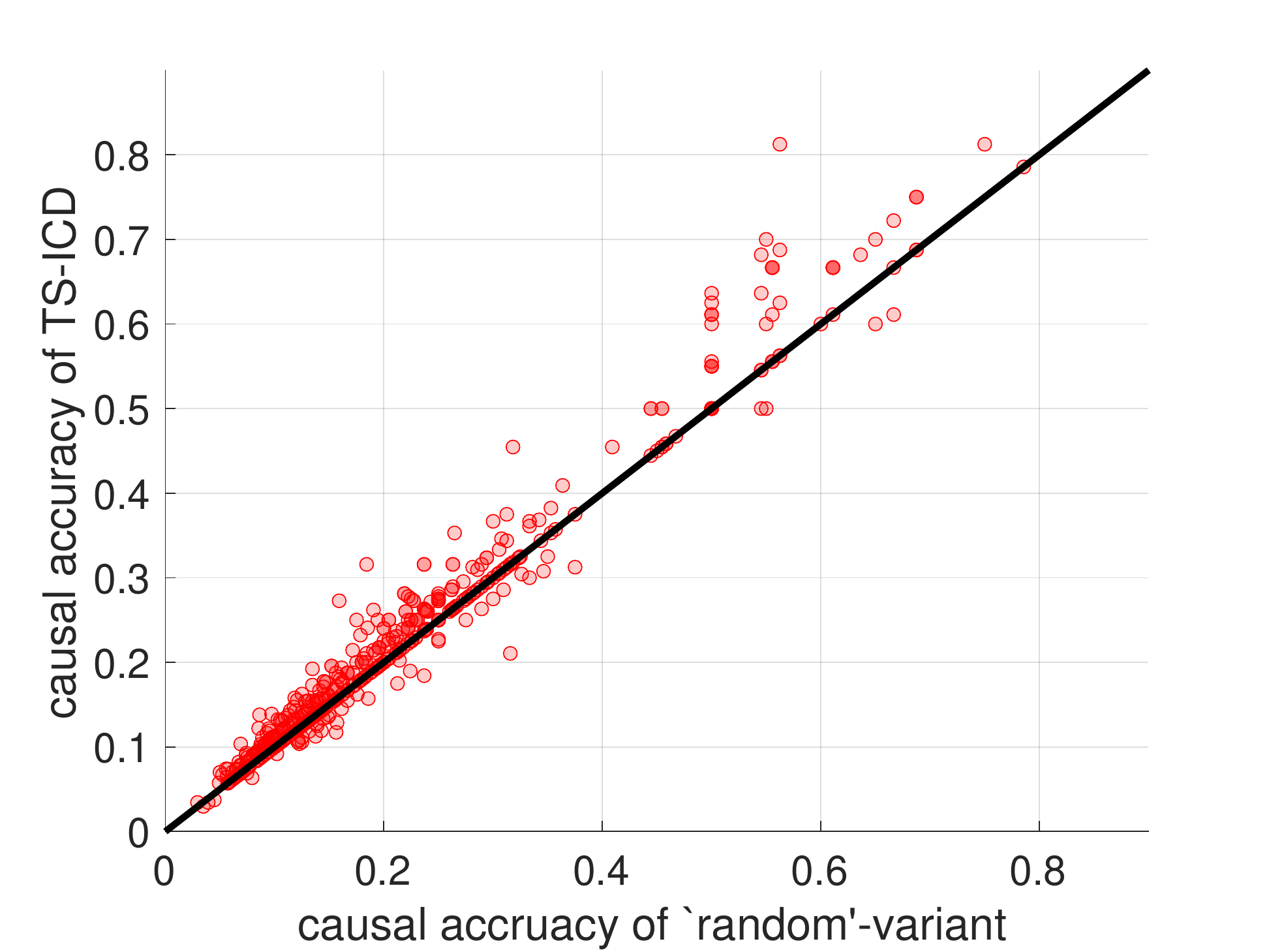}\\
(c)\includegraphics[width=0.90\columnwidth]{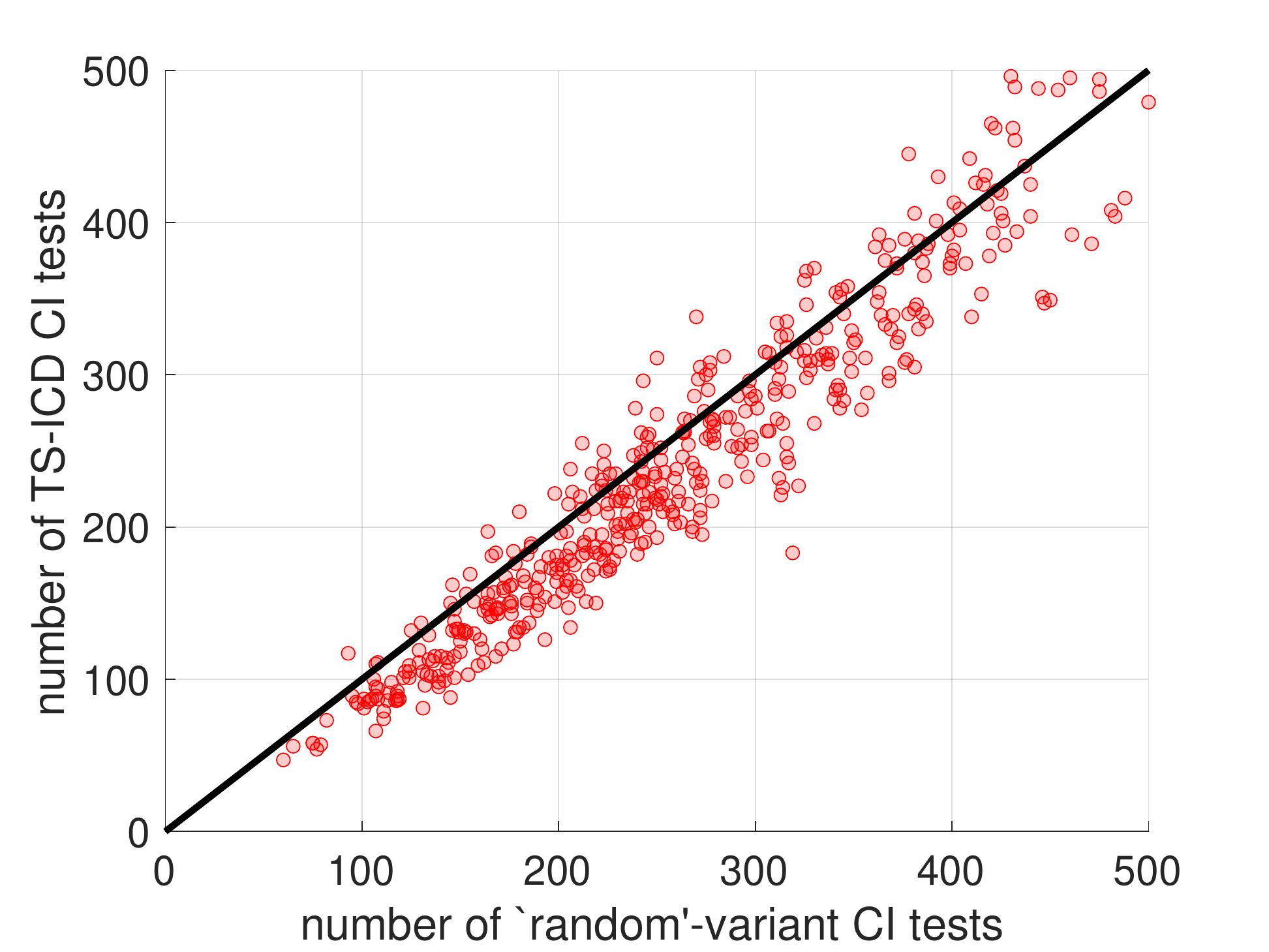}
\caption{What if the order of edges for CI testing is random? Scatter plots for `random' (order of edges) variant vs. TS-ICD for: (a) F1-score for skeleton, (b) causal accuracy, (c) number of CI tests. Difference in skeleton F1-scores is not statistically significant. TS-ICD achieves higher causal accuracy and requires fewer CI tests, with statistical significance.}
\label{fig:random_scatter}
\end{figure}

A, `random order' is equivalent to applying ICD algorithm \citep{rohekar2021iterative} to time-series data, and a reversed order is application of the TS-ICD algorithm using a reversed refinement order, where long-range relations are refined after shorter range ones.

\begin{figure}[b!]
\centering
(a)\includegraphics[width=0.90\columnwidth]{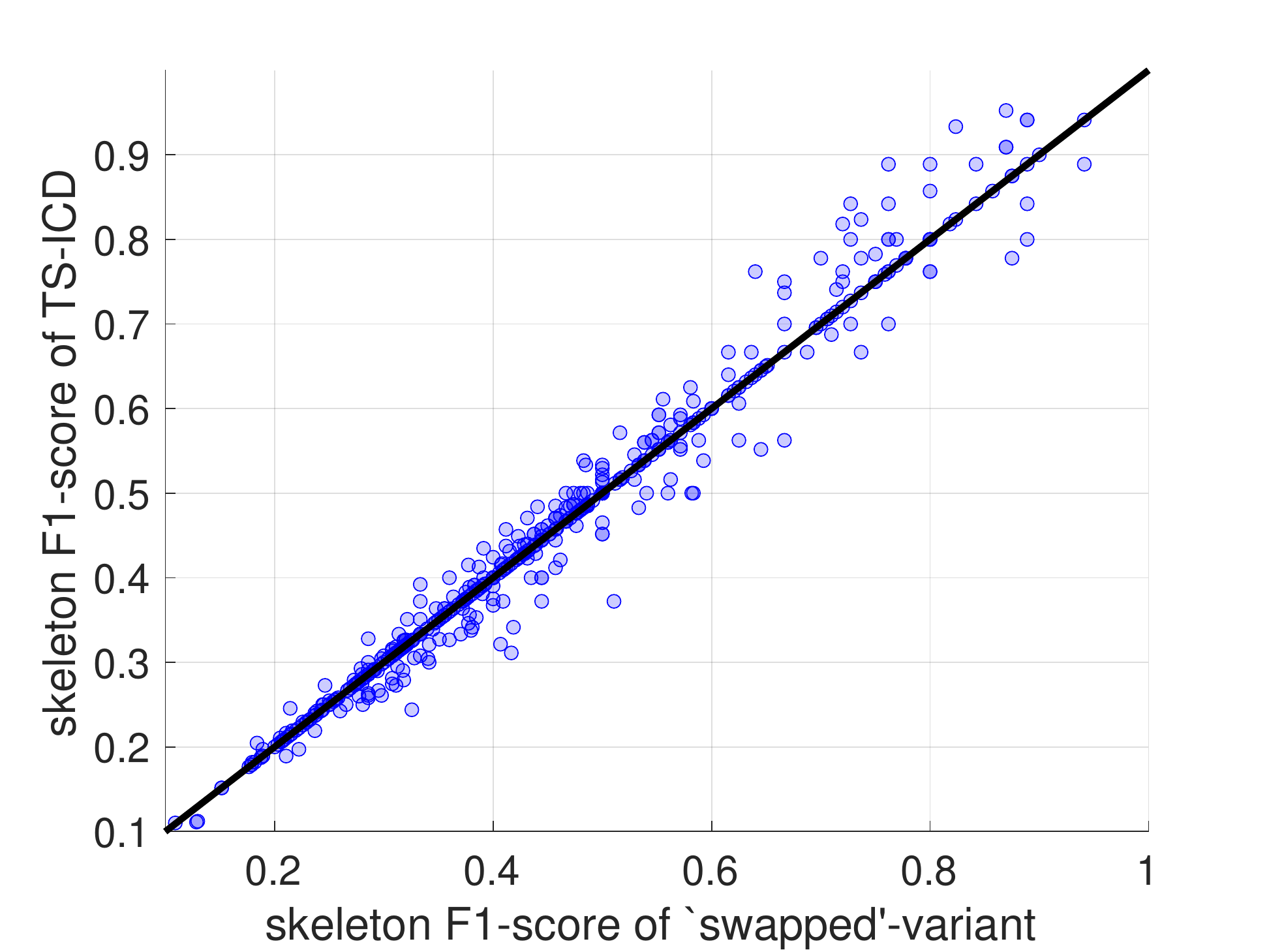}
(b)\includegraphics[width=0.90\columnwidth]{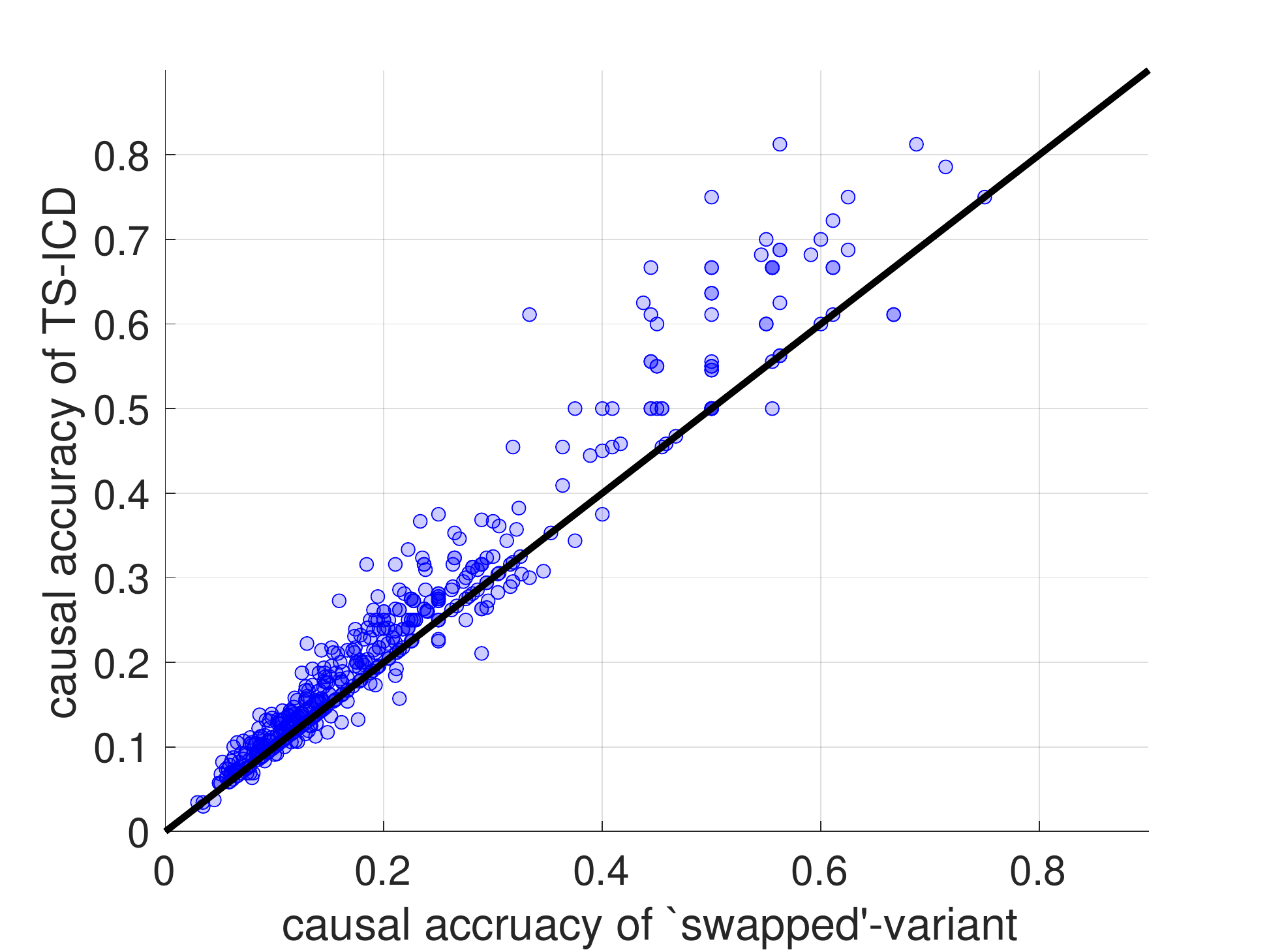}
(c)\includegraphics[width=0.90\columnwidth]{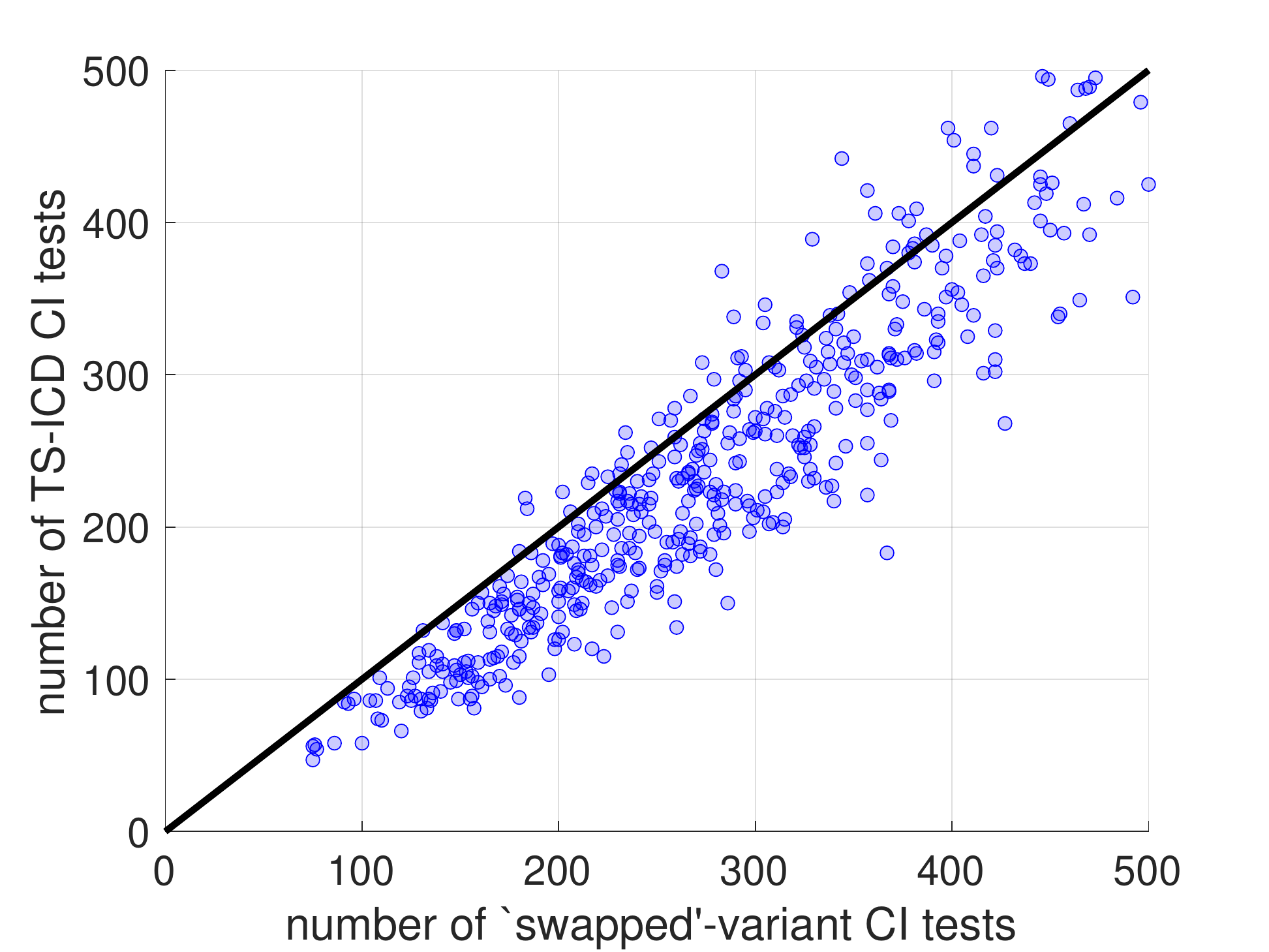}
\caption{What if we reverse the order of edges for CI-testing used by TS-ICD? Scatter plots for `swapped' variant vs. TS-ICD for: (a) F1-score for skeleton, (b) causal accuracy, (c) number of CI tests. Difference in skeleton F1-scores is not statistically significant. TS-ICD achieves higher causal accuracy and requires fewer CI tests, with statistical significance.}
\label{fig:swapped_scatter}
\end{figure}

\section{Synthetic Data: False Negatives, False Positives, Precision, and Recall\label{apx:fprfnr}}

In this section, we provide the false-positive ratio, false-negative ratio, precision, and recall for the graph-skeletons learned in experiments described in \secref{sec:synthetic_experiment} and \tabref{tab:parcorr500}. Results for the linear-Gaussian model are provided in \tabref{tab:fr-lin-gauss} and for the discrete (binary) nodes model in \tabref{tab:fr-discrete}. For the linear-Gaussian model, SVAR-FCI and LPCMCI achieved the lowest false-positive ratios and highest precision of identified edges. For the discrete model, SVAR-FCI achieved the lowest false-positive ratios and highest precision.
For both models, TS-ICD achieved the lowest false-negative and highest recall. This indicates that TS-ICD compared to the tested algorithms recovers more correct edges (fewer erroneous omissions).  
In extended experiments with data sizes $T=200$ and $T=1000$, the relative advantages and disadvantages of algorithms, and conclusions were similar. 

\begin{table}
\centering
\small
\begin{tabular}{lcc|cc}
    \toprule
    Algorithm & FNR & FPR & Precision & Recall \\
    \midrule
    SVAR-FCI & 0.7576 & \textbf{0.0000} & \textbf{1.0000} & 0.2424 \\
    (MAD) & {\small(0.1496)} & {\small(0.0047)} & {\small(0.0600)} & {\small(0.1496)} \\
    \midrule
    SVAR-RFCI & 0.7308 & 0.0085 & 0.9000 & 0.2692 \\
    (MAD) & {\small(0.1482)} & {\small(0.0084)} & {\small(0.0922)} & {\small(0.1482)} \\
    \midrule
    LPCMCI & 0.6087 & \textbf{0.0000} & \textbf{1.0000} &  0.3913\\
    (MAD) & {\small(0.1816)} & {\small(0.0070)} & {\small(0.0577)} & {\small(0.1816)} \\
    \midrule
    TS-ICD & \textbf{0.5600} & 0.1504 & 0.6000 & \textbf{0.4400} \\
    (MAD) & {\small(0.0707)} & {\small(0.0540)} & {\small(0.1813)} & {\small(0.1504)} \\
    \bottomrule
\end{tabular}

\caption{Performance for the linear-Gaussian model.}
\label{tab:fr-lin-gauss}
\end{table}

\begin{table}
\centering
\small
\begin{tabular}{lcc|cc}
    \toprule
    Algorithm & FNR & FPR & Precision & Recall \\
    \midrule
    SVAR-FCI & 0.5964 & \textbf{0.0009} & \textbf{0.9857} & 0.4036 \\
    (MAD) & {\small(0.2768)} & {\small(0.0016)} & {\small(0.0257)} & {\small(0.2768)} \\
    \midrule
    SVAR-RFCI & 0.5768 & 0.0017 & 0.9746 & 0.4232 \\
    (MAD) & {\small(0.2974)} & {\small(0.0027)} & {\small(0.0406)} & {\small(0.2974)} \\
    \midrule
    LPCMCI & 0.6392 & 0.0700 & 0.5716 &  0.3608\\
    (MAD) & {\small(0.2763)} & {\small(0.0518)} & {\small(0.2555)} & {\small(0.2763)} \\
    \midrule
    TS-ICD & \textbf{0.5328} & 0.0172 & 0.8311 & \textbf{0.4672} \\
    (MAD) & {\small(0.2805)} & {\small(0.0120)} & {\small(0.1129)} & {\small(0.2805)} \\
    \bottomrule
\end{tabular}

\caption{Performance for the discrete (binary nodes) model.}
\label{tab:fr-discrete}
\end{table}


\end{document}